\documentclass[10pt,twocolumn,letterpaper]{article}

\usepackage{cvpr}
\usepackage{times}
\usepackage{epsfig}
\usepackage{graphicx}
\usepackage{amsmath}
\usepackage{amssymb}

\usepackage{amsthm,amsfonts}
\usepackage{algorithm}
\usepackage{algpseudocode}
\usepackage{tabularx}
\usepackage{tabulary}
\usepackage{array}
\usepackage{float}
\usepackage{stfloats}
\usepackage{placeins}
\usepackage{stackengine}
\DeclareMathOperator*{\argmax}{argmax}

\newtheorem{lemma}{Proposition}

\usepackage{float}
\newfloat{algorithm}{t}{lop}

\makeatletter
\renewcommand{\paragraph}{ \@startsection{paragraph}{4} {\z@}{1.25ex \@plus 1ex \@minus .2ex}{-1em} {\normalfont\normalsize\bfseries} }
\makeatother

\newcommand\blfootnote[1]{%
  \begingroup
  \renewcommand\thefootnote{}\footnote{#1}%
  \addtocounter{footnote}{-1}%
  \endgroup
}

\usepackage[pagebackref=true,breaklinks=true,letterpaper=true,colorlinks,bookmarks=false]{hyperref}

\cvprfinalcopy 

\def\cvprPaperID{8} 

\ifcvprfinal\fi
\begin{document}

\title{Cluster-Wise Ratio Tests for Fast Camera Localization}
\author{Ra{\'u}l D{\'i}az, Charless C. Fowlkes\\
Computer Science Department, University of California, Irvine\\
{\tt\small \{rdiazgar,fowlkes\}@uci.edu}
}

\maketitle

\begin{abstract}

Feature point matching for camera localization suffers from scalability
problems.  Even when feature descriptors associated with 3D scene points are
locally unique, as coverage grows, similar or repeated features become
increasingly common. As a result, the standard distance ratio-test used to
identify reliable image feature points is overly restrictive and rejects many
good candidate matches. We propose a simple coarse-to-fine strategy that uses
conservative approximations to robust local ratio-tests that can be computed
efficiently using global approximate k-nearest neighbor search. We treat these
forward matches as votes in camera pose space and use them to prioritize
back-matching within candidate camera pose clusters, exploiting feature
co-visibility captured by the 3D model camera pose graph.  This approach
achieves state-of-the-art camera pose estimation results on a variety of
benchmarks, outperforming several methods that use more complicated data
structures and that make more restrictive assumptions on camera pose. We 
carry out diagnostic analyses on a difficult test dataset containing globally
repetitive structure which suggest our approach successfully adapts to the
challenges of large-scale pose estimation.

\blfootnote{This work is supported by NSF grants IIS-1618806 and IIS-1253538.}
\end{abstract}


\vspace{-0.5cm} 
\section{Introduction}

In this paper we consider the problem of estimating the full 6DOF camera pose
of a query image with respect to a large-scale 3D model such as those obtained
from a Structure-from-Motion (SfM) pipeline
\cite{Snavely2010,wu2013towards,moulon2013adaptive,schoenberger2016sfm}.  A
typical approach is to detect distinctive 2D feature points in a query image
and perform correspondence search against feature descriptors associated with
3D points obtained from the SfM reconstruction. This initial matching is
performed in descriptor space (e.g., SIFT \cite{Lowe2004} or SURF
\cite{bay2008speeded}) using an approximate k-nearest neighbor search
implementation \cite{ann, flann}.  Candidate 2D-3D correspondences are then
further filtered using robust fitting techniques (e.g., RANSAC variants
\cite{ransac, mlesac, acransac}) to identify inliers and the final camera pose
estimated using an algebraic PnP solver and non-linear refinement. Camera
pose estimation is a fundamental building block in many computer vision algorithms
(e.g., incremental bundle adjustment), can provide strong constraints on object
recognition (see e.g., \cite{Wang_2015_CVPR,mvbs_iccv}), and is useful in robotics
applications such as autonomous driving and navigation.

Unfortunately, the performance of standard camera localization pipelines degrades
as the size of the 3D model grows.  Finding good correspondences becomes 
difficult in the large-scale setting due to two factors. First, standard 2D-to-3D 
\textit{forward} matching is likely to accept bad correspondences of a query feature 
with the model since the feature space becomes cluttered with similar descriptors 
from completely different locations. Standard heuristics for identifying
distinctive matches, such as the distance ratio-test of Lowe \cite{Lowe2004},
which compares the distance to the nearest-neighbor point descriptor with that
of the second-nearest neighbor, fail due to proximity of other model feature
descriptors.  Second, increasingly noisy correspondences obtained from the
matching stage drives up the runtime of the robust pose estimation step, whose
complexity typically grows exponentially with the number of outliers. These
difficulties are particularly evident in large urban environments, where
repeated structure is common and local features become less distinctive
\cite{Torii_2013,arandjelovic2014}.

\begin{figure*}[ht]
  \centering
  \includegraphics[width=0.7\paperwidth]{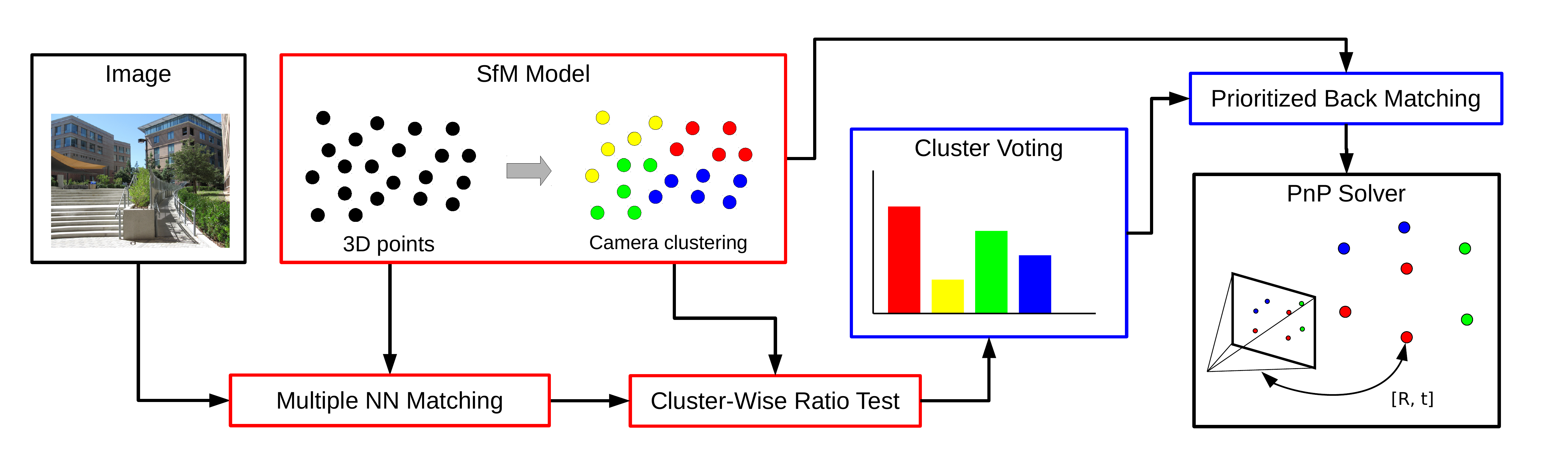} 
  \caption[Large Scale Camera Localization Pipeline]{Overview of our camera
  pose-estimation pipeline. We first exploit multiple nearest neighbor search and
  camera pose clustering to identify candidate feature correspondences
  (red boxes, described in Section \ref{sec:fwd-pass}).  We then utilize 
  co-visibility to expand this set and prioritize back-matching of model
  features (blue boxes, described in Section \ref{sec:bck-pass}).
}
  \label{fig:loc-pipe}
\end{figure*}

\paragraph{Related Work:}
These problems are well known and have been approached in several ways in the
literature. Works such as \cite{li2010location,Li} focus on generating a
simplified 3D model that contains only a representative subset of distinctive
model points. With a smaller model and prioritized search, it becomes possible
to replace the traditional approach of 2D-to-3D \textit{forward} matching, with
3D-to-2D \textit{back} matching, allowing the ratio test to be performed in the
sparser feature space of the query image. 

 
An alternative to removing points from the model is to cluster and quantize
model point feature descriptors. \cite{Sattler2011} use vocabulary trees to
speed up forward matching by assigning each model point and each query feature
to a vocabulary word, yielding faster runtimes since the vocabulary size is
generally smaller than the model point cloud.  A linear search for the first
and second nearest neighbors is performed within each word bin, and a ratio
test filters out non-distinct correspondences.  \cite{Sattler2012} use active
search in the vocabulary tree to prioritize back matching of 3D points close to
those that have already been found and terminate early as soon as a sufficient
number of matches have been identified.

A very different approach is taken in the works of \cite{zeisl2015camera,
svarm2014accurate}. Camera localization is framed as a Hough voting procedure,
where the geometric properties of SIFT (scale and orientation) 
provide approximate information about likely camera pose from individual
point correspondences. By using focal length and camera orientation priors,
each 2D-to-3D match casts a vote into the intersection of a visibility cone and
a hypothesized ground-plane. Orientation and model co-visibility are further
used to filter out unlikely matches, rapidly identifying the potential camera
locations. 

\paragraph{Our Contribution:}
Inspired by this prior work, we propose a fast, simple method for camera
localization that scales well to large models with globally repeated structure.
Our approach avoids complicated data structures and makes no hard {\em a
priori} assumptions on camera pose (e.g., gravity direction of the camera). Our
basic insight is to utilize a coarse-to-fine approach that rapidly narrows down
the region of camera pose space associated with the query image. Specifically,
we formulate a linear time voting process over camera pose space by assigning
each single model view to an individual camera pose bin. 
This voting allows us to identify model views likely to overlap the query
image and to prioritize back matching of those views against it
while exploiting co-visibility constraints and local ratio testing.

Figure \ref{fig:loc-pipe} gives on overview of our pipeline.  Our first
contribution (Section \ref{sec:fwd-pass}) is to introduce and analyze two
ratio-tests that can be used to find distinctive matches in a pool of
candidates produced by global k-nearest neighbor search (kNN).  Our second
contribution (Section \ref{sec:bck-pass}) uses these forward matches as votes
to prioritize back matching of model images against the query image. Extensive
experimental evaluation (Section \ref{sec:results}) suggests this approach
scales well and outperforms existing methods on several pose-estimation
benchmarks.


\section{Ratio Tests for Global Matching} \label{sec:fwd-pass}

Forward-matching of query image points against a model is effective when
the model is small. In such models, approximate nearest-neighbors are 
often true correspondences and ratio-testing is effective at discarding
bad matches.  In this section we first establish that clustering the 
model into smaller sub-models and performing forward-matching within each
cluster is sufficient to achieve good performance for large models (Section
\ref{sec:exhaustive}).  We then describe how to efficiently approximate
exhaustive cluster-wise matching by global forward-matching using
approximations to the local ratio test (Section \ref{sec:ratio-tests}) followed
by back-matching.

\subsection{Clustering and Exhaustive Local Matching} \label{sec:exhaustive}

A naive approach to solving camera localization at large scale is to simply
divide the 3D model into small pieces (clusters) and perform matching and
robust PnP pose estimation for each cluster.  This avoids the problems of
global feature repetition and difficulties of high density in the feature
space.  However, this is infeasible from a computational point of view as it
requires building a nearest-neighbor data structure for each cluster and
matching to each cluster separately at test time.  Consider a kd-tree, where
searching for a match in a set with $N$ descriptors is logarithmic in the set
size: $O(log(N))$.  If we divide the model into $|\mathcal{C}|=N/S$ clusters of
constant size $S$, execution time is dominated by the number of clusters which
grows linearly in the model size, $O(|\mathcal{C}|
log(\frac{N}{|\mathcal{C}|})) = O(\frac{N}{S} log(S)) > O(log(N))$.  While not
practical at scale, we take this {\em exhaustive local matching} approach as a
gold-standard baseline for evaluating our coarse-to-fine approach.

\paragraph{Exhaustive Local Matching is effective but slow:}

To evaluate clustering and local matching, we use the Eng-Quad dataset from
\cite{tsu-gis}, and build two SfM models using COLMAP
\cite{schoenberger2016sfm}.  The first model contains only the training image
set, while a second model bundles both the training and test images and is used
for evaluating localization accuracy.  We geo-register the resulting
reconstructions with a GIS model so that the scale of the SfM model is
approximately metric. 5129 training images of the 6402 were bundled, and 520
out of 570 test images were additionally bundled in the test model.  The
resulting point cloud has 579,859 3D points and 2,901,885 feature descriptors.
We refer to these descriptors as \textit{views} of the point. 

To generate clusterings of the model, we construct a scene matrix $S$ whose
$(i,j)$ entry contains the number of points that image pair $I_{i},I_{j}$ share
in the SfM model. We performed spectral clustering \cite{shi2000normalized} on
the scene matrix using the 50 largest eigenvectors and produce three different
granularities: no clustering at all (purely global), 50 clusters, and 500
clusters. To evaluate exhaustive local matching, we matched a query image
against every cluster and select the one that produces the smallest
localization error.  For matching to a cluster, we use FLANN \cite{flann} to
find the first and second NN of each query point and apply a standard ratio
test with a $\tau=0.7$ threshold.  We ran RANSAC on each set of candidate
cluster correspondences using a P3P solver \cite{kneip2011novel} and a focal
length prior based on the image EXIF metadata.  Similar to
\cite{li2010location}, an image is considered to be successfully matched if it
has at least $12$ inlier correspondences with a reprojection error less than
$\epsilon=6px$. 

Table \ref{tab:gs-table} shows that exhaustive local matching within each
cluster performs much better than global matching, with lower median error and
fewer failures.  However, the execution time grows roughly linearly with
respect to the number of clusters, motivating our coarse-to-fine strategy.

\begin{table}[t]
\begin{center}
{\scriptsize
\setlength{\tabcolsep}{3pt}
\begin{tabular}{|c||c|c|c|c||c|c|c|}
\hline
\#clusters & \#images & \#inliers & ratio & error [m] & fwd [s] & RNSC [s]& total [s] \\ \hline \hline
1 (global) & 463      & 94        & 0.57  & 0.64   & 0.833   & 0.129   & 0.962    \\ \hline 
50         & 512      & 66        & 0.54  & 0.45   & 13.10   & 43.62   & 56.822   \\ \hline
500        & 517      & 51        & 0.49  & 0.29   & 80.23   & 523.69  & 603.52   \\ \hline 
\end{tabular}
}
\end{center}
   \caption[Gold Standard Localization using Model Clustering]{Results on Eng-Quad using 
   a standard localization framework applied to model clusters. Performing localization 
   separately in each cluster improves the number of localized cameras and the median error 
   accuracy, at the expense of longer runtimes due to exhaustive matching.
   We also report the number of inliers and inlier ratio, as well as forward matching,
   RANSAC, and total times.}
\label{tab:gs-table}
\end{table}

\floatplacement{algorithm}{t}
\begin{algorithm}
\small
    \caption{Global Forward Matching}
    \begin{algorithmic}
        \State {\bf INPUT:} Query features $Q$, Model features $\mathcal{V}$, 
        NN search depth $k$, ratio test threshold $\tau$, match count threshold $N_{F}$
         
        \State $\mathcal{M} = \emptyset$
        \While{$|\mathcal{M}| < N_{F}$}
        \State $q = $ random sample from $Q$
            \State $\{v_{1}, ..., v_{k+1}\} = $ kNN($q,\mathcal{V},k+1$)
            \State $\alpha_{q} = \frac{\lVert q-v_{1} \rVert}{\lVert q-v_{k+1} \rVert}$
            \If{$\alpha_{q} \leq \tau$}
                \State $\mathcal{M} = \mathcal{M} \cup \{(q,v_{1}), ..., (q,v_{k}) \} $
            \EndIf
        \EndWhile
        \State \textbf{return} $\mathcal{M}$
    \end{algorithmic}
    \label{alg:propreg}
\end{algorithm}

\begin{figure}
  \centering
  \includegraphics[width=0.9\linewidth]{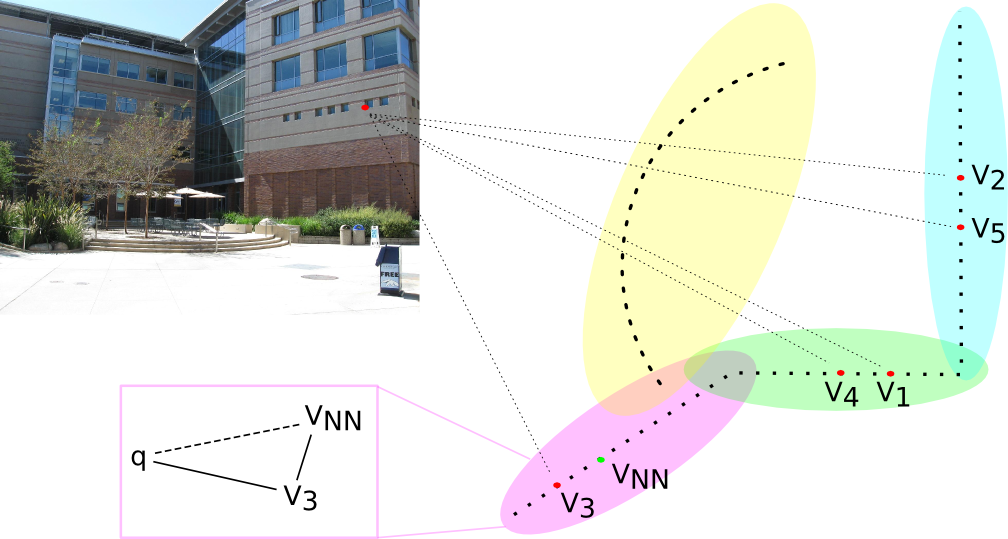}
  \caption[Cluster-wise ratio test]{Cluster-wise ratio testing.
  Model views are divided into clusters. For a query feature $q$ in the image
  (red dot), we search up to 5 nearest neighbors. Within clusters containing
  two or more matches, we can perform a standard local 1-ratio test within the
  cluster (e.g. $v_1$ and $v_4$ are the first and second NN in the green
  cluster).  For the singleton $v_3$ in the pink cluster, we use an alternate
  \textit{t-ratio} test  (Eq. \ref{eqn:trt}) based on the matched view's nearest
  neighbor $v_{NN}$ (rather than the query point's true second nearest neighbor
  within the cluster).}
  \label{fig:fwd_rt}
\end{figure}

\subsection{Local Ratio Tests for Global Matches} \label{sec:ratio-tests}

\textit{How can we get the benefits of local cluster-wise matching while
maintaining the computational cost associated with a single global
nearest-neighbor search?} Cluster-wise matching considers a nearest-neighbor
per-cluster for each query point.  To try and recover this larger pool of
candidate correspondences using global search, we propose to retrieve the global
top k nearest-neighbors for each query point.  Fortunately, approximate kNN
searches are not substantially more costly since those points typically live in
adjacent leaves of the kd-tree (which must be explored even for a 1-NN
retrieval). A larger set of candidate matches can address the problem of
repeated structure by retrieving the set of multiple scene points that might
correspond to a query point. However, it also results in a k-fold increase in
outliers which we now address.

We define a \textit{view} $v \in \mathcal{V}$ as the 2D point observation of a
3D point $p \in \mathcal{P}$ in a particular model image $I \in \mathcal{I}$.
Given a camera pose clustering $\mathcal{C}$ of the SfM model images, we assign
the view descriptors of each image to their corresponding cluster $c \in
\mathcal{C}$. Note that these clusters divide images in disjoint groups, but
they do share common points, as a 3D point can have multiple views belonging to
images assigned to different clusters.  For a query image $I$ with query
features $Q$, we search for $k$ approximate nearest neighbors using a global
kd-tree structure built from all views $\mathcal{V}$. 



\paragraph{Global k-ratio tests:}
We start with a conservative global ratio-test (Algorithm \ref{alg:propreg})
to prune candidate matches by comparing the distance ratio of the first and
$k+1$ nearest neighbor retrieved, as proposed by \cite{zamir2014image}. If the
ratio is greater than threshold $\tau$, we drop the query point.  Otherwise,
all $k$ nearest neighbors pairs $\{(q,v_{1}), ...,(q,v_{k})\}$ are included in
the set of putative correspondences $\mathcal{M}$.  This \textit{global} ratio
test is much more conservative than the standard first vs second NN test. In
the remainder of this paper, we will refer to this global test as
\textit{k-ratio}, defined formally as $\frac{\lVert q-v_{1} \rVert}{\lVert
q-v_{k+1} \rVert} \leq \tau$.  The standard first versus second NN test will be
referred as \textit{1-ratio}. 

\begin{lemma}
If a candidate match fails the global k-ratio test, it also fails the local
1-ratio test.
\end{lemma}

\begin{proof}
Let $\{v_{c_1},v_{c_2}\} \subset \{v_{1}, ..., v_{k}\}$ be the first and second
\textit{local} nearest neighbors of a particular query feature $q$. Since the global set
$\{v_{1}, ..., v_{k}\}$ is sorted by ascending distance, this implies that $\lVert
q-v_{c_2} \rVert \leq \lVert q-v_{k+1} \rVert$, and $\lVert q-v_{c_1} \rVert
\geq \lVert q-v_{1} \rVert$. Formally,
\begin{equation}
   \frac{\lVert q-v_{c_1} \rVert}{\lVert q-v_{c_2} \rVert} \geq
   \frac{\lVert q-v_{c_1} \rVert}{\lVert q-v_{k+1} \rVert} \geq 
   \frac{\lVert q-v_{1} \rVert}{\lVert q-v_{k+1} \rVert}
\end{equation}

Hence, the local 1-ratio will always be equal or greater than the global
k-ratio. This guarantees that any correspondence rejected by the k-ratio test
would also have failed the local 1-ratio test. A correspondence passing the
k-ratio test might not pass the local 1-ratio test, so the local 1-ratio 
test is a more stringent criteria.
\end{proof}

\paragraph{Cluster-wise ratio tests:}
After the initial global filtering, we would like to perform local ratio
testing within each cluster. When more than two candidate matches for a query
point belong to the same cluster, we can simply \textit{re-rank} them and apply
a standard 1-ratio test.  For example, suppose two global matches $(q,v_{2})$
and $(q,v_{4})$ which are the second and fourth global NN of the query feature
$q$ fall into the same cluster.  If $v_{2}$ and $v_{4}$ are views of distinct 3D
points, then they are necessarily the first $(q,v_{c_1})$ and second
$(q,v_{c_2})$ local nearest-neighbors of $q$ in that cluster (see Figure
\ref{fig:fwd_rt}). Any lower-ranked matches within the cluster can be ignored
and the 1-ratio test applied to this pair.

When only a single global match falls within a cluster we can no longer perform
an exact local 1-ratio test since we do not have immediate access to the 2nd
nearest neighbor within that cluster.  Instead we develop a bound based on the
triangle inequality to define an alternate test for such cases which we refer
to as the \textit{t-ratio} test. 

Given a local correspondence $v_{c_1} \in c$, we define $v_{NN} =
kNN(v_{c_1},c,1)$ as the nearest neighbor of view  $v_{c_1}$ in the feature
space defined by cluster $c$. Since $v_{NN}$ is obtained purely from training
data, we can pre-compute it offline and access it at test time.  We define the
\textit{t-ratio} test as:
\begin{equation}
   \frac{\lVert q-v_{c_1} \rVert}{\lVert q-v_{c_1} \rVert + \lVert v_{c_1} - v_{NN} \rVert} \leq \tau
   \label{eqn:trt}
\end{equation}
Although we missed the local 2nd nearest neighbor in the global search, the
distance $\lVert v_{c_1} - v_{NN} \rVert$ provides useful information on how
far away the 2nd nearest neighbor might be. 

\begin{lemma}
If a candidate match fails the t-ratio test, it also fails the local 1-ratio test.
\end{lemma}

\begin{proof}
Let $q$ be a query feature, $v_{c_1}$ and $v_{c_2}$ the fist and second local
nearest neighbors in a cluster $c$, and $v_{NN} = kNN(v_{c_1},c,1)$. We can
bound the distance to the second nearest neighbor by the inequalities:
\begin{equation}
   \lVert q-v_{c_2} \rVert \leq 
   \lVert q-v_{NN} \rVert \leq 
  \lVert q-v_{c_1} \rVert + \lVert v_{c_1}-v_{NN} \rVert 
\end{equation}
where the first inequality holds since 
$\lVert q-v_{c_2} \rVert \leq \lVert q-v \rVert \quad \forall \, v \in c \setminus v_{c_1}$, 
and the second holds by the triangle inequality. Thus, 
\begin{equation}
   \frac{\lVert q-v_{c_1} \rVert}{\lVert q-v_{c_2} \rVert} \geq
   \frac{\lVert q-v_{c_1} \rVert}{\lVert q-v_{c_1} \rVert + \lVert v_{c_1} - v_{NN} \rVert}  
\end{equation}
Consequently, a singleton match that fails the t-ratio test will always fail
the local 1-ratio test. The t-ratio test thus only filters correspondences that
would have failed the local ratio test if $v_{c_2}$ was available.
\end{proof}

\paragraph{Back-matching and fitting:}
To provide additional robustness to outliers, we can \textit{back match} views
(model feature point descriptors) which were indicated as candidate
correspondences from the forward matching.  For any such candidate matching
view, we search for the first and second nearest neighbor matches using a
kd-tree built over the query image features and apply the 1-ratio test.  We
then select as the final set of correspondences the intersection of pairs
$(q,v)$ that passed the forward and back matching process.  These pairs are
cluster-wise \textit{best buddies} \cite{dekel2015best}, since each $q$ and $v$
of a pair are both discriminative features in the query and model feature
space. 

\floatplacement{algorithm}{t}
\begin{algorithm}
\small
    \caption{Cluster-Wise Ratio Test}
    \begin{algorithmic}
        \State {\bf INPUT:} matches $\mathcal{M}$, clusters $\mathcal{C}$, threshold $\tau$
        \State $\mathcal{M_{F}} = \emptyset $
	\For{ $c \in \mathcal{C}$ }
	  \For{$(q,v_{c_1}) \in \mathcal{M}$ with $v_{c_1}\in c$}
	    \If{$(q,v_{c_2}) \in \mathcal{M}$ with $v_{c_2}\in c$}
	      \State $\alpha_{(q,c)} = \frac{\lVert q-v_{c_1} \rVert}
                                                 {\lVert q-v_{c_2} \rVert}$
      \Comment{local 1-ratio test}
	    \Else
              \State $v_{NN} = kNN(v_{c_1},c,1)$
	      \State $\alpha_{(q,c)} = \frac{\lVert q-v_{c_1} \rVert}
		     {\lVert q-v_{c_1} \rVert + \lVert v_{c_1}-v_{NN} \rVert}$
      \Comment{t-ratio test}
	    \EndIf
	    \If{$\alpha_{(q,c)}\leq \tau$}
		\State $\mathcal{M_{F}} = \mathcal{M_{F}} \cup (q,v_{c_1})$\
	    \EndIf
	  \EndFor
	\EndFor
    \State \textbf{return} $\mathcal{M_{F}}$
    \end{algorithmic}
    \label{alg:fwdrt}
\end{algorithm}

\begin{table}
\begin{center}
{\scriptsize
\setlength{\tabcolsep}{2pt}
\begin{tabular}{|c||c|c|c|c||c|c|c|c|c|}
\hline
\#clusters & \#imgs & \#inl & ratio & err. & fwd [s] & RT [s] & bck [s] & RNSC [s]& total [s] \\ \hline \hline
1          & 481    & 115   & 0.74  & 0.69 & 0.821   & 0.008  & 0.021   & 0.046  & 0.895 \\ \hline
50         & 477    & 127   & 0.59  & 0.66 & 0.818   & 0.008  & 0.028   & 0.061  & 0.915\\ \hline
500        & 480    & 133   & 0.56  & 0.61 & 0.821   & 0.009  & 0.038   & 0.066  & 0.934\\ \hline
5129       & 482    & 136   & 0.55  & 0.62 & 0.833   & 0.009  & 0.048   & 0.070  & 0.961\\ \hline
\end{tabular}
}
\end{center}

\caption[Localization using multiple nearest neighbors]{Quantitative results on
the 520 test image set using the proposed localization framework of algorithm
\ref{alg:fwdrt} and best-buddy filtering. We used 5 nearest neighbors in the
k-NN search. We evaluated four spatial subdivisions, including a finest
clustering in which each camera in the model is considered a single cluster.
Localization accuracy is competitive with exhaustive local matching with
achieving runtimes comparable to global matching.
}
\label{tab:base-table}
\end{table}

\subsection{Cluster-wise ratio-tests are effective and fast} \label{sec:effective}

The cluster-wise ratio test, defined in Algorithm \ref{alg:fwdrt}, prunes a
large number of non-discriminative correspondences while still maintaining the
locally unique matches.  The complexity of this algorithm is linear in the
number of forward correspondences $N_{F}$.  For every local NN $v_{c_1}$, we
simply look for its second NN pair $v_{c_2}$ within the list of $k$ nearest
neighbors. The list of intra-cluster nearest neighbors is simply a view-to-$v_{NN}$
vector that can be pre-computed offline and accessed at constant time, similar to
vocabulary-based methods that store view-to-word assignments. Hence, at most $N_{F}$
ratio tests will be performed.

We evaluated this cluster-wise approach using the same settings as our gold
standard baseline experiment. We added a finer division of the model,
consisting of atomic clusters with a single image each. Table
\ref{tab:base-table} shows the localization performance on these different
granularities.  A single global cluster gives surprisingly good results in the
number of localized cameras, although it provides worse camera position
results. This is due to the restrictiveness of the ratio test in denser search
spaces, yielding fewer inliers and missing some discriminative correspondences
that would improve results. As we increase the number of clusters, the
localization errors are reduced (8 cm on average) thanks to the cluster-wise
ratio test which provides more high confidence matches (at the expensive of
longer RANSAC runtimes). We obtain best results using the finest clustering (a
single model camera per cluster), successfully localizing 482 images.  Compared
to the gold-standard of Table \ref{tab:gs-table}, our strategy is 
competitive, by only dropping 5\% in localization performance while being three
orders of magnitude faster. Moreover, the finest single-image clusters provide
the best result we can avoid running any complex clustering method (e.g.,
spectral clustering).  We use single-image clusters in the remainder of the
paper. 

\section{Accelerating Matching by Pose Voting} \label{sec:bck-pass}
As Table \ref{tab:gs-table} suggests, with appropriate cluster-wise testing,
forward matching now constitutes the primary computational bottleneck.
Short of simplifying the model (e.g., as pursued by \cite{Li,li2010location}), 
how might we further accelerate the matching process?  A natural strategy is 
to carry out forward matching incrementally and stop as soon as we have a 
sufficient number of matches to guarantee a good result. From this perspective, 
we can view forward matching as ``voting'' for the location of the query camera.
Unlike \cite{zeisl2015camera,svarm2014accurate} where votes were cast into a 
uniformly binned camera translation space, we use each model camera pose as
a putative bin to cast our votes (also used in \cite{Sattler2015}).  We avoid
additional data structures like vocabulary trees in favor of storing a simple
but effective view-to-$v_{NN}$ vector that enforces local uniqueness.  Once we
have accumulated enough votes to narrow down the camera pose to a few candidate
clusters, we can terminate forward matching and carry out back matching with
little loss in accuracy.

\subsection{Coarse localization using cluster matching}

To analyze how many votes are needed to determine a good localization, we frame
the problem as that of \textit{location recognition}
\cite{Sattler_2016,Torii_2015,baatz2012large,schindler2007city}, namely
producing a short ranked list of model images that depict the same general 
location as the query image. We follow the evaluation procedure
of~\cite{cao2013graph}, reporting if there exists at least one image among the
\textit{top-k} images that shares 12 or more fundamental matrix inliers. We
benchmark performance on two datasets: Eng-Quad and Dubrovnik
\cite{li2010location}.

\begin{table}[t]
\begin{center}
{\scriptsize

\begin{tabular}{|c||c|c|c|c||c|}
\multicolumn{6}{c}{Dubrovnik - 800 test images}  \\ \hline
Method             & top-1   & top-2   & top-5   & top-10  & Time [s] \\ \hline \hline
$N_{F}=50$         & 99.00\% & 99.38\% & 99.62\% & 99.88\% & 0.048    \\ \hline
$N_{F}=100$        & 99.62\% & 99.75\% & 99.88\% & 99.88\% & 0.085    \\ \hline
$N_{F}=200$        & 100\%   & 100\%   & 100\%   & 100\%   & 0.157    \\ \hline

\end{tabular}

\vspace{2 mm}

\begin{tabular}{|c||c|c|c|c||c|}
\multicolumn{6}{c}{Eng-Quad - 520 test images}  \\ \hline
Method             & top-1    & top-2   & top-5   & top-10  & Time [s]\\ \hline \hline
$N_{F}=50$         & 83.85\%  & 85.96\% & 88.46\% & 88.65\% & 0.064 \\ \hline
$N_{F}=100$        & 84.62\%  & 86.54\% & 89.62\% & 90.96\% & 0.125 \\ \hline
$N_{F}=200$        & 85.77\%  & 87.69\% & 90.38\% & 91.35\% & 0.242 \\ \hline
$N_{F}=500$        & 86.15\%  & 88.27\% & 90.96\% & 91.35\% & 0.502 \\ \hline
All features       & 86.92\%  & 89.23\% & 91.15\% & 91.92\% & 0.833 \\ \hline
\end{tabular} 

}
\end{center}

\caption[Location recognition of Eng-Quad and Dubrovnik]{We achieve perfect
location recognition results on the Dubrovnik dataset using a random subset of
200 query features that pass the k-ratio and cluster-wise ratio tests,
suggesting that our approach successfully finds local discriminative
correspondences for all 800 test images.  We also obtain good results in the
more challenging Eng-Quad dataset, recognizing 478 ($91.92\%$) images. This
agrees with the baseline results obtained in Table \ref{tab:base-table}. 
}
\label{tab:locreg-table}
\end{table}

The results in Table \ref{tab:locreg-table} are inspiring. Algorithm
\ref{alg:fwdrt} is able to recognize the location of all 800 test images in the
Dubrovnik dataset using 200 random features passing the k-ratio test. Results
on the more challenging Eng-Quad dataset provide almost $92\%$ accuracy on
recognizing the landmarks of the 520 query images for which we have a ground
truth pose. Importantly, a random subset of a few hundred query features
achieves nearly as good recognition results as using all image features (a
query image usually has 5,000 to 10,000 features). This suggests that the
forward matching can be terminated early while still maintaining good
localization performance. 

\subsection{Prioritized Back Matching}

Determining the correct model image only provides rough camera location and
additional work is needed to estimate the precise camera pose.  To reap the
computational benefits of subsampling, we thus modify our framework slightly.
We use forward matching with a subset of $N_{F}$ query features in order to
identify likely model images.  We then perform back matching within candidate
images in order to expand the set of matches used for fine camera pose
estimation. This back matching is carried out using a greedy prioritized search
over images ranked by votes and further exploits co-visibility information
encoded in the SfM model to find additional distinctive matches that were not
identified during the forward (sub-sampled) matching.

\floatplacement{algorithm}{t}
\begin{algorithm}
\small
    \caption{Prioritized Back Matching}
    \begin{algorithmic}
        \State{{\bf INPUT:} forward matches $\mathcal{M_{F}}$, clustering $\mathcal{C}$, 
           query features $Q$, threshold $\tau$, minimum number of matches $N_{B}$,
           scene graph $G$}
        \State $H_{c} = |(q,v) \in \mathcal{M_{F}} : v \in c| \quad \forall c \quad$   \Comment{Cast votes}
        \State $\mathcal{M_{B}} = \emptyset$, $VC = \emptyset$
        \While{ $(|\mathcal{M_{B}}| < N_{B}) \wedge (|VC| \leq 20) $ }
            \State $c^{*} = \argmax_{c \notin VC} H$
            \State $\mathcal{M}_{c^{*}} = \emptyset$ 
            \For{$v \in c^{*}$} 
              \State $q_{1},q_{2} = kNN(v,Q,2)$ \Comment{Back match $v$ to query}
              \State $\alpha_{v} = \frac{\lVert v-q_{1} \rVert}{\lVert v-q_{2} \rVert}$
              \If{$\alpha_{v} \leq \tau$}
                 \State $\mathcal{M}_{c^{*}} = \mathcal{M}_{c^{*}} \cup (q_{1},v) $
              \EndIf
            \EndFor
            \State $\mathcal{M_{B}} = \mathcal{M_{B}} \cup \mathcal{M}_{c^{*}}$ 
            \If{$|\mathcal{M}_{c^{*}}| \geq 12$} 
	            \For{$(q,v) \in \mathcal{M}_{c^{*}}$}
                \For{$c' \in \mathcal{C}$ with $v \in c'$}
  	              \State $H_{c'} = H_{c'} + 1$  \Comment{Update votes} 
                \EndFor 
              \EndFor
            \EndIf
            \State $VC = VC \cup c^{*}$ \Comment{Update visited clusters (images)}
        \EndWhile
        \State \textbf{return} $\mathcal{M_{B}}$
    \end{algorithmic}
    \label{alg:bckrt}
    \vspace{-0.04in}
\end{algorithm}

Algorithm \ref{alg:bckrt} describes our back-matching approach. Given the
forward matches found using Algorithm \ref{alg:fwdrt}, we select the most voted
model image $c^{*}$ and back-match all of its views against the query image
using the standard 1-ratio test with threshold $\tau$.
The correspondences $M_{c^{*}}$ found are added to the pool of back matched
pairs $\mathcal{M_{B}}$ used for the fine pose estimation. These back matches
are also treated as votes.  We use the SfM model's camera-point visibility
graph $G$, to cast votes for other images that observe the same views as in
$M_{c^{*}}$.  These new votes increase the likelihood that neighboring images
are selected for subsequent rounds of back-matching. To avoid introducing
noise into the voting process, we only allow a back-matched image to cast
votes if it depicts the same location (i.e., returns 12 or more matches).  The
algorithm terminates when $\mathcal{M_{B}}$ is large enough to guarantee a good
camera localization, or a certain total number of images have been
back-matched.

\floatplacement{algorithm}{t}
\begin{algorithm}
\small
\caption{Camera Localization}
\begin{algorithmic}
        \State{{\bf INPUT:} Query features $Q$, Model features $\mathcal{V}$, 
        co-visibility graph $G$, camera clusters $\mathcal{C}$, NN search depth $k$,  
        ratio test threshold $\tau$, match count thresholds $N_{F},N_{B}$, 
        projection error threshold $\epsilon$} 
        \State $\mathcal{M} = $ GLOBAL-FORWARD-MATCH($Q,\mathcal{V},k,N_{F},\tau$)
        \State $\mathcal{M_{F}} = $ CLUSTER-WISE-RATIO-TEST($\mathcal{M},\mathcal{C},\tau$)
        \State $\mathcal{M_{B}} = $ PRIORITY-BACK-MATCH($\mathcal{M_{F}}, N_{B}, G, \tau$)
        \State $I_{P},\delta = $ ROBUSTFITTING($\mathcal{M_{B}},\epsilon$)
        \If{$|\delta \leq \epsilon| \geq 12$}
                \State \textbf{return} Camera Pose $I_{P}$ 
        \Else
                \State \textbf{return} Error - Pose not found
        \EndIf
    \end{algorithmic}
    \label{alg:toplevel}
    \vspace{-0.04in}
\end{algorithm}

\section{Benchmark Evaluation} \label{sec:results}

We evaluated our approach (Algorithm \ref{alg:toplevel}) on
three different datasets: Eng-Quad, Dubrovnik, and Rome.  Rome is a large
dataset of 15,179 training and 1,000 test images.
Dubrovnik is a popular 6,044 training and 800 test image dataset whose SfM 
model is roughly aligned to geographic coordinates, allowing for quantitative
metric evaluation. While Eng-Quad has fewer images, it is perhaps the most challenging due
to the presence of strongly repeated structures in the modern architectural
designs it depicts.  When using P3P, we used EXIF metadata for Eng-Quad test
images and ground-truth focal lengths from the SfM models for Dubrovnik and
Rome. We also briefly analyze results on the city-wide SF-0 dataset \cite{Li}.

\begin{table*}[t]
\begin{center}
{\scriptsize

\setlength{\tabcolsep}{3pt}

\begin{tabular}{cc}

\begin{tabular}{|l||c|c|c|ccc|cc|c|}
\multicolumn{10}{c}{Dubrovnik (Original) - 800 test images}  \\ \hline
 & & & & \multicolumn{3}{c|}{error [m]} & & & \\ 
Method & \#images & \#inliers & ratio & Q1 & median & Q3 &  $<$18.3m &$>$400m&time [s] \\ \hline \hline
Sattler \cite{Sattler2011} & 783.9 &$\leq$100 & -    & 0.4   & 1.4  & 5.9  & 685 & 16& 0.31 \\ \hline
Sattler \cite{Sattler2012} & 795.5 &$\leq$200 & -    & 0.4   & 1.4  & 5.3  & 704 & 9 & 0.25 \\ \hline
Zeisl \cite{zeisl2015camera}& 796  & -        & -    & 0.19  & 0.56 & 2.09 & 744 & 7 & 3.78 \\ \hline
Svarm \cite{svarm2016city} & 798   & -        & -    & -     & 0.56 & -    & 771 & 3 & 5.06 \\ \hline
Ours (P3P)                 & 800   & 358      & 0.65 & 1.09  & 7.92 & 27.76& 550 & 10& 0.62 \\ \hline
Ours (P4Pf)                & 800   & 468      & 0.79 & 0.55  & 1.64 & 6.02 & 694 & 15& 0.62 \\ \hline

\multicolumn{10}{c}{}  \\

\multicolumn{10}{c}{Dubrovnik (Corrected) - 777 test images}  \\ \hline
 & & & & \multicolumn{3}{c|}{error [m]} & & & \\ 
Method & \#images & \#inliers & ratio & Q1 & median & Q3 &  $<$18.3m &$>$400m&time [s] \\ \hline \hline
Sattler \cite{Sattler2011} & 771 & 70  & 0.72 & 0.57  & 1.44 & 4.61 & 707 & 1 & 2.58 \\ \hline
Sattler \cite{Sattler2012} & 775 & 69  & 0.74 & 0.59  & 1.58 & 4.91 & 705 & 4 & 0.75 \\ \hline
Ours (P3P)                 & 777 & 591 & 0.88 & 0.33  & 0.66 & 1.60 & 759 & 1 & 0.48 \\ \hline
Ours (P4Pf)                & 776 & 589 & 0.88 & 0.47  & 1.11 & 3.32 & 720 & 3 & 0.48 \\ \hline
\end{tabular}

&

\begin{tabular}{c}

\begin{tabular}{|l||c|c|c|ccc|c|}
\multicolumn{8}{c}{Eng-Quad - 520 test images}  \\ \hline
 & & & & \multicolumn{3}{c|}{error [m]} & \\ 
Method & \#images & \#inliers & ratio  & Q1  & median & Q3 & time [s] \\ \hline \hline
Sattler \cite{Sattler2011} & 402 & 43  & 0.49 & 0.61 & 2.01 & 7.51 & 1.52 \\ \hline
Sattler \cite{Sattler2012} & 457 & 43  & 0.58 & 0.46 & 1.93 & 7.62 & 0.32 \\ \hline
Ours (P3P)                 & 509 & 112 & 0.66 & 0.33 & 0.67 & 1.47 & 0.69 \\ \hline
Ours (P4Pf)                & 504 & 115 & 0.68 & 0.65 & 1.88 & 5.76 & 0.85 \\ \hline
\end{tabular}

\\ \\

\begin{tabular}{|c||c|c|c|c|}
\multicolumn{5}{c}{Rome - 1000 test images}  \\ \hline
Method & \#images & \#inliers & ratio & time [s]\\ \hline \hline
P2F \cite{li2010location}  & 924    & -          & -    & 0.87 \\ \hline
Sattler \cite{Sattler2011} & 976.90 & $\leq100$  & -    & 0.29 \\ \hline
Sattler \cite{Sattler2012} & 991    & $\leq200$  & -    & 0.28 \\ \hline
Ours (P3P)                 & 999    & 281        & 0.54 & 0.75 \\ \hline
Ours (P4Pf)                & 1000   & 458        & 0.83 & 0.74 \\ \hline
\end{tabular}

\end{tabular}

\end{tabular}

}
\end{center}

\caption[Camera localization performance]{Quantitative results of our method 
compared to related methods for camera pose estimation.}
\label{tab:reg-table}
\vspace{-0.1in}
\end{table*}


\paragraph{Dubrovnik correctness:} After carefully analyzing the original
models provided for Dubrovnik, we found that {\it the test set ground truth was
often wrong}, with extremely large focal lengths and misaligned 2D-3D correspondences.
This in turn resulted in large errors in camera
location and poor alignment between projection of 3D points and the
corresponding 2D features. These problems are evident in results published
elsewhere. For example, \cite{sattler2014sampling} report better results using
P4Pf \cite{bujnak2008general} than using P3P with the given ``true'' focal lengths.
This is contrary to what should be expected: knowing the ground-truth focal
length (P3P) should outperform joint estimation of pose and focal length
(P4Pf).  Examples are shown in the supplementary material.

For this reason, we rebuilt a new version of the Dubrovnik ``ground-truth''
model using the same set of keypoints provided for the original dataset and the
excellent SfM package COLMAP \cite{schoenberger2016sfm}.  We aligned the new
model with the original one using a RANSAC-based Procrustes analysis so that
the scale is approximately metric. After alignment, only 3853 of the recovered
6844 images were located within 3 meters from their original position in the
model, further validating our concerns. Our reconstruction provided
ground-truth for 777 of the 800 query images.

\begin{figure}[b]
  \setlength{\tabcolsep}{3pt}
  \centering
  \begin{tabular}{cc}
  \includegraphics[width=0.48\linewidth]{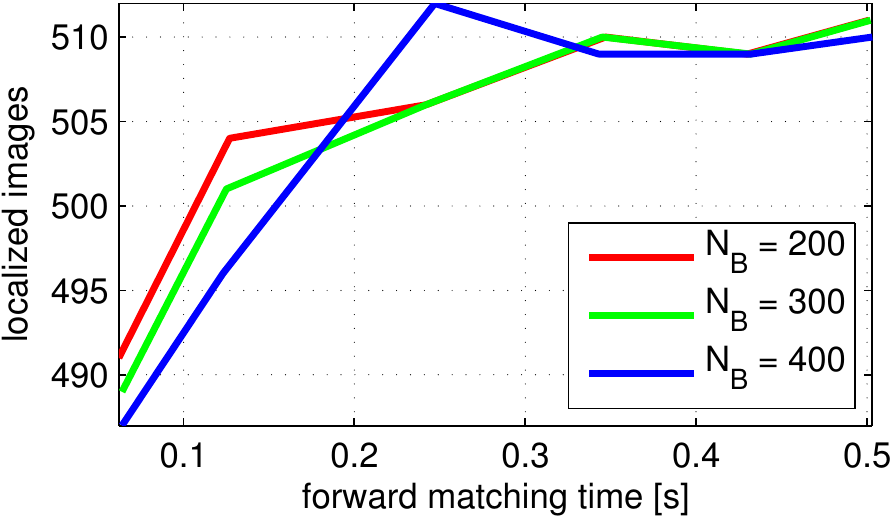} &
  \includegraphics[width=0.48\linewidth]{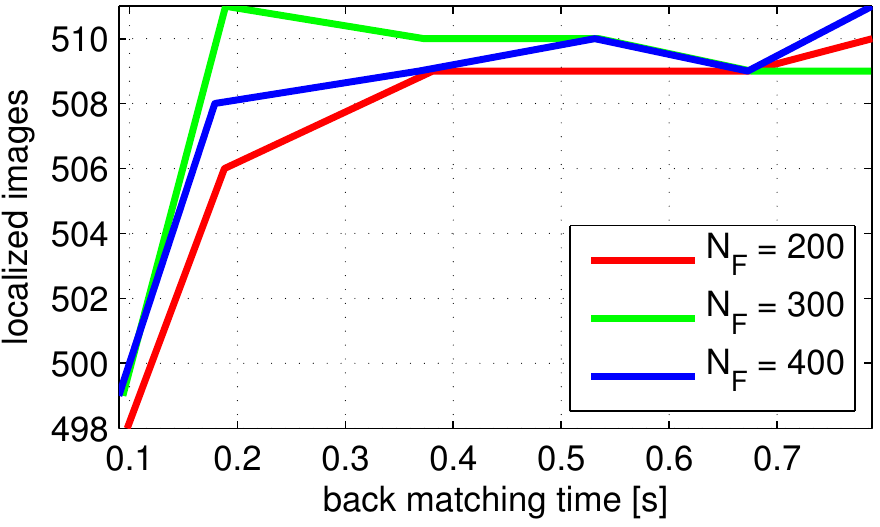}
  \end{tabular}
  \caption[Performance trade-off between forward and back matching]{Anytime
  performance: querying a small number of features dramatically reduces runtime
  without a major loss in localization performance. The forward subsampling
  does not affect rough localization significantly and stabilizes after 0.3
  seconds (left) regardless of the $N_{B}$ value. Similarly, localization
  quickly plateaus after 0.2 seconds at back matching time for different values
  of $N_{F}$ (right).}
  \label{fig:anytime}
\end{figure}

\begin{figure*}[t]
\centering
  \setlength{\tabcolsep}{1pt}
  \def\stackalignment{l}
  \begin{tabular}{cc@{\hskip 10pt}cc}
  \topinset{ \includegraphics[width=0.09\paperwidth]{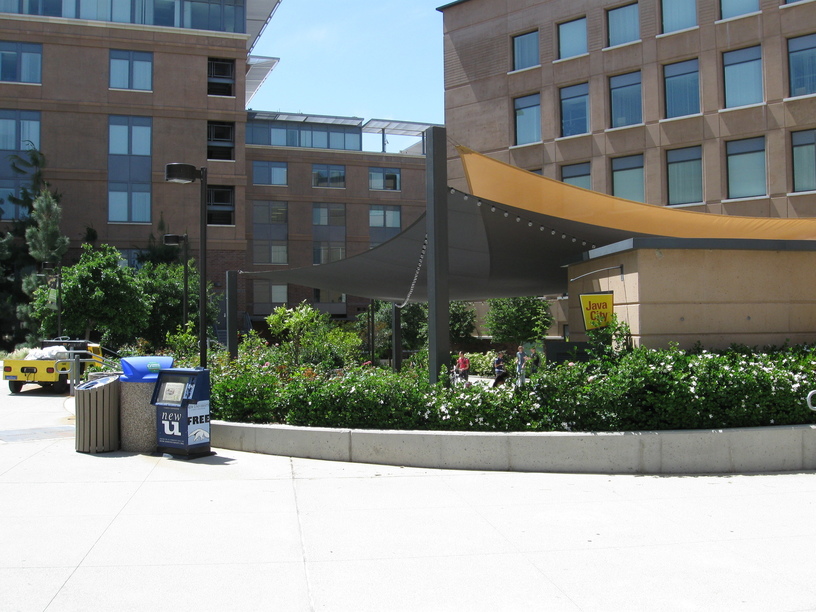} }
  { \includegraphics[width=0.2\paperwidth]{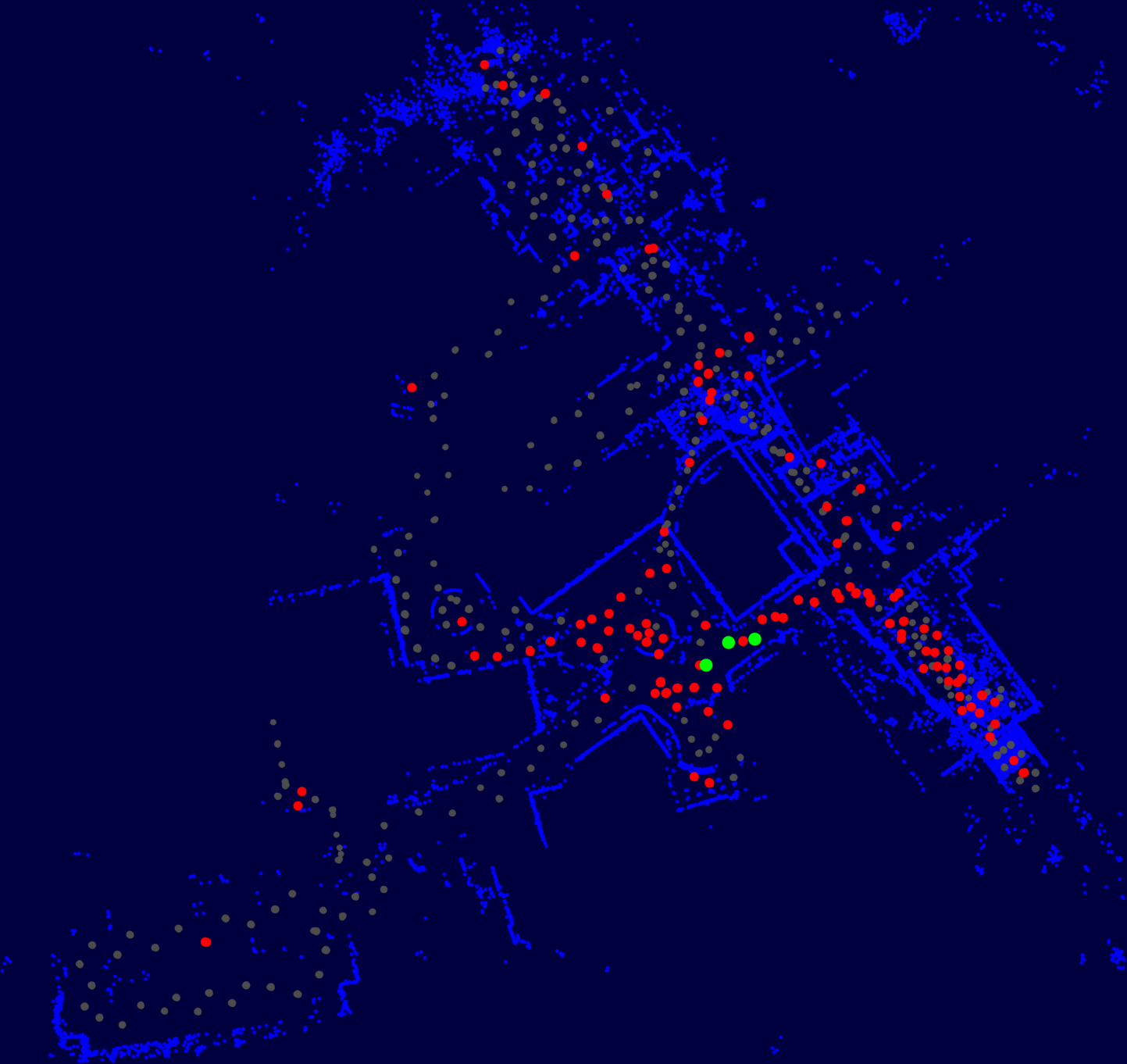} }{25pt}{-5pt}&
  \includegraphics[width=0.2\paperwidth]{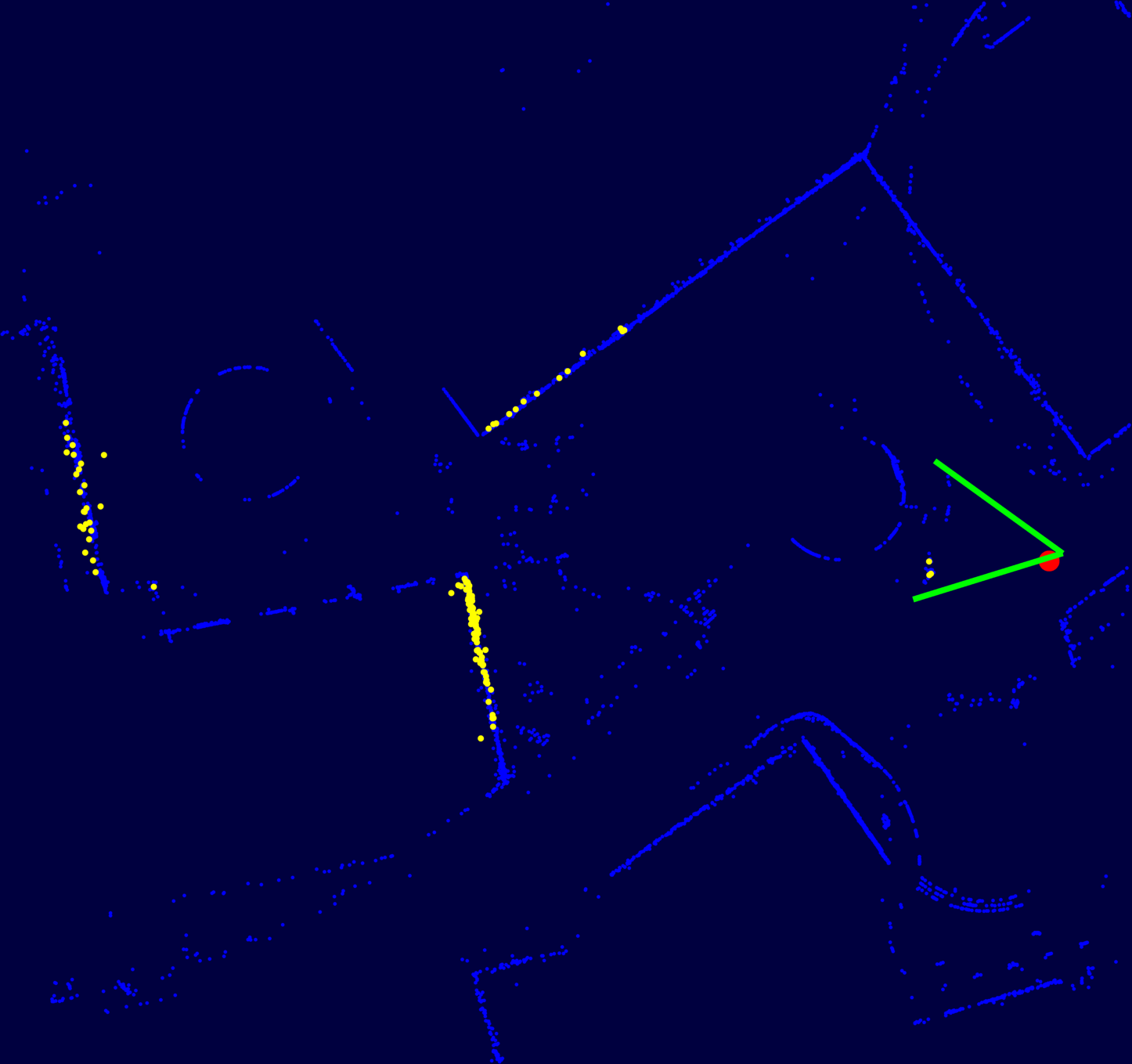}
  &
  \topinset{\includegraphics[width=0.06\paperwidth]{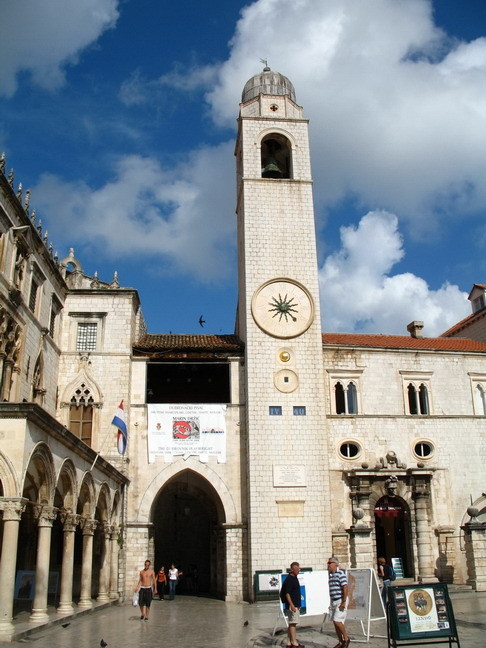}}
  { \includegraphics[width=0.16\paperwidth]{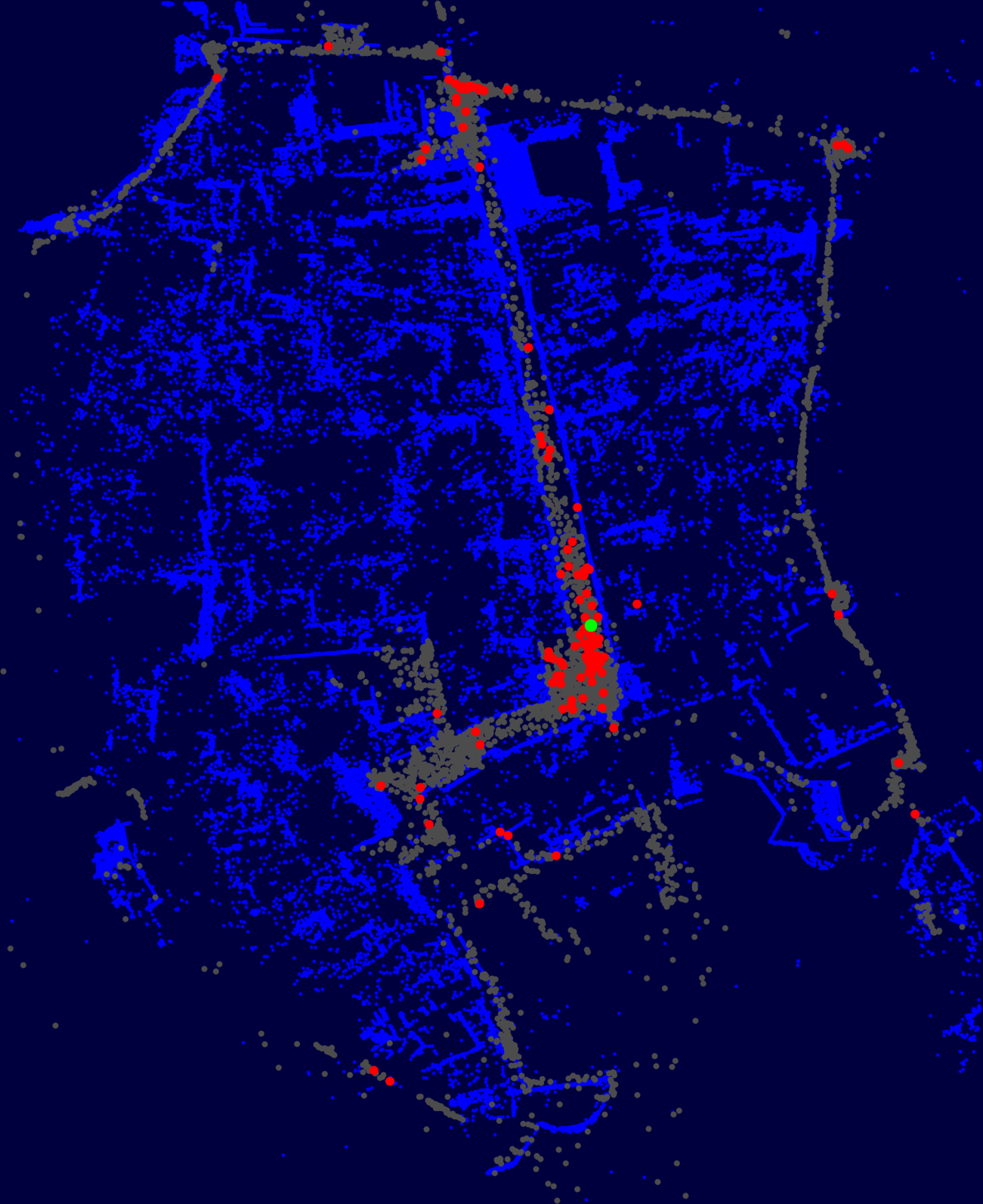} }{30pt}{-5pt} &
  \includegraphics[width=0.16\paperwidth]{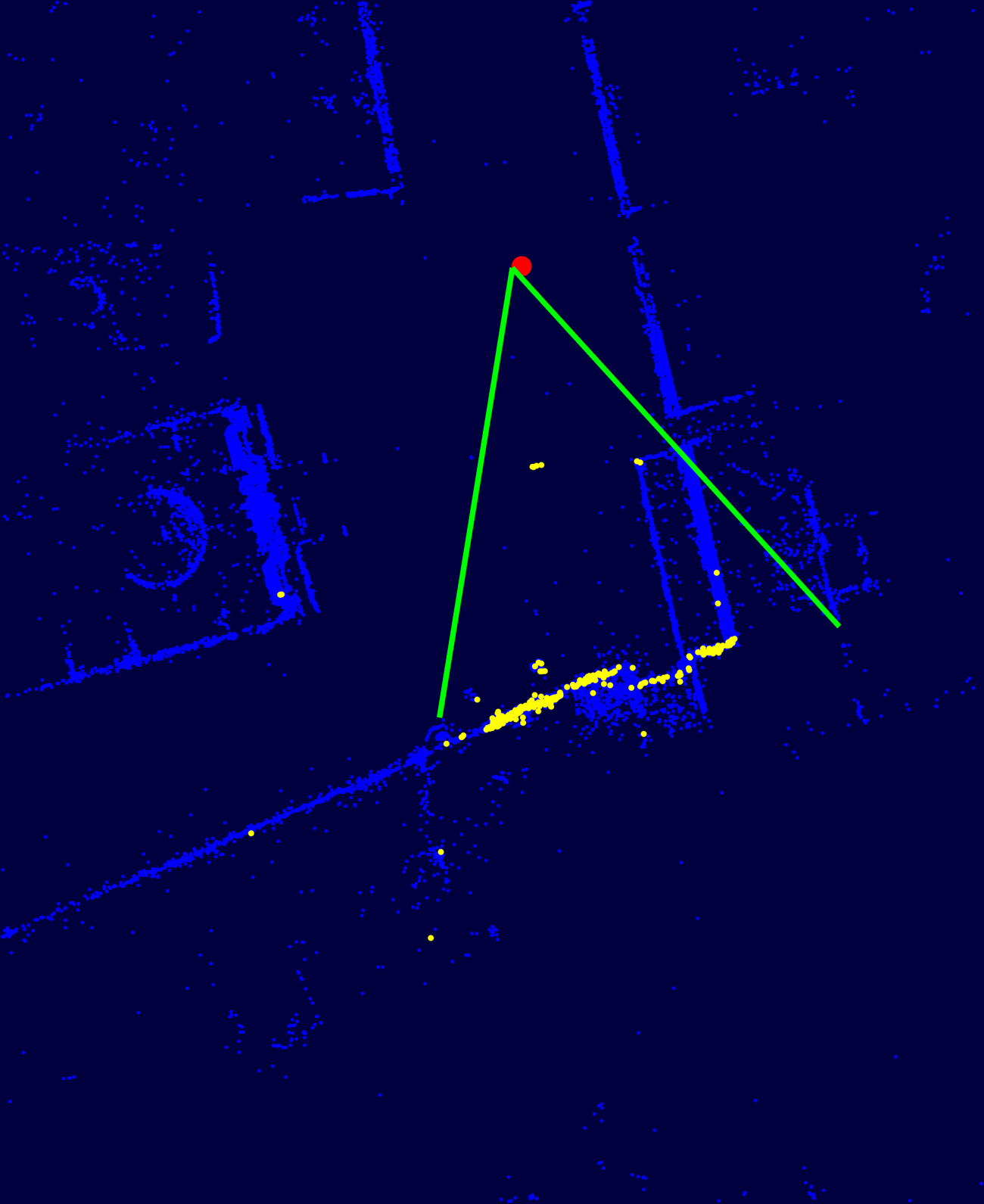}\\

  \topinset{\includegraphics[width=0.09\paperwidth]{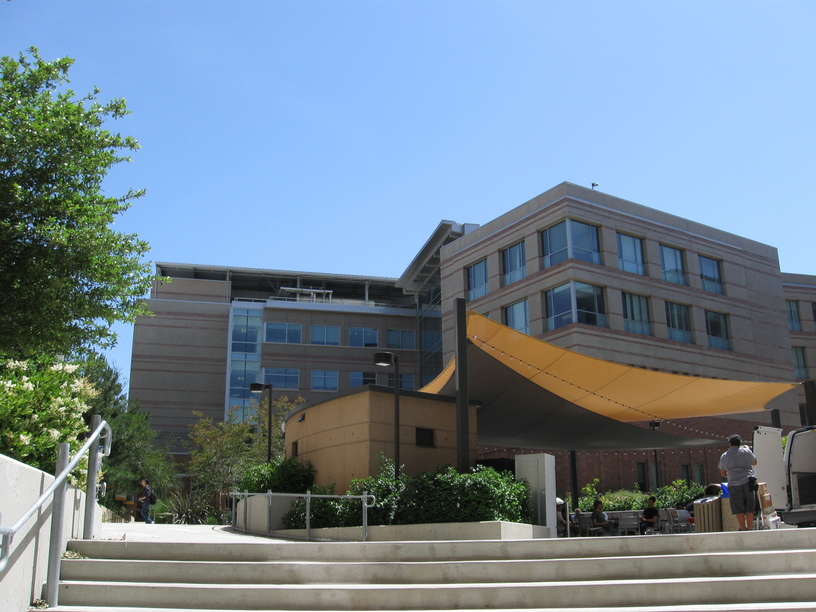}} 
  { \includegraphics[width=0.2\paperwidth]{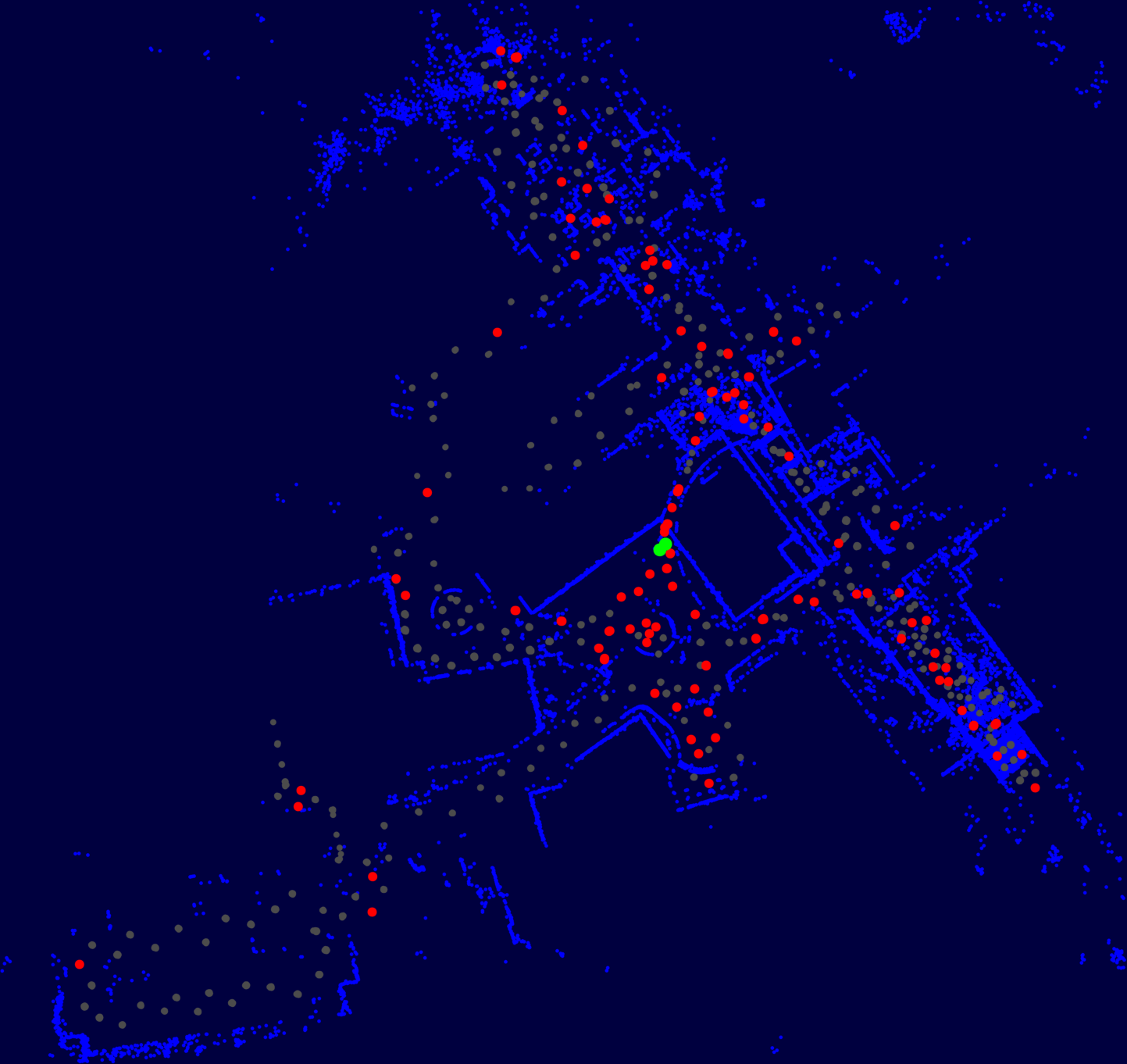} }{25pt}{-5pt} &
  \includegraphics[width=0.2\paperwidth]{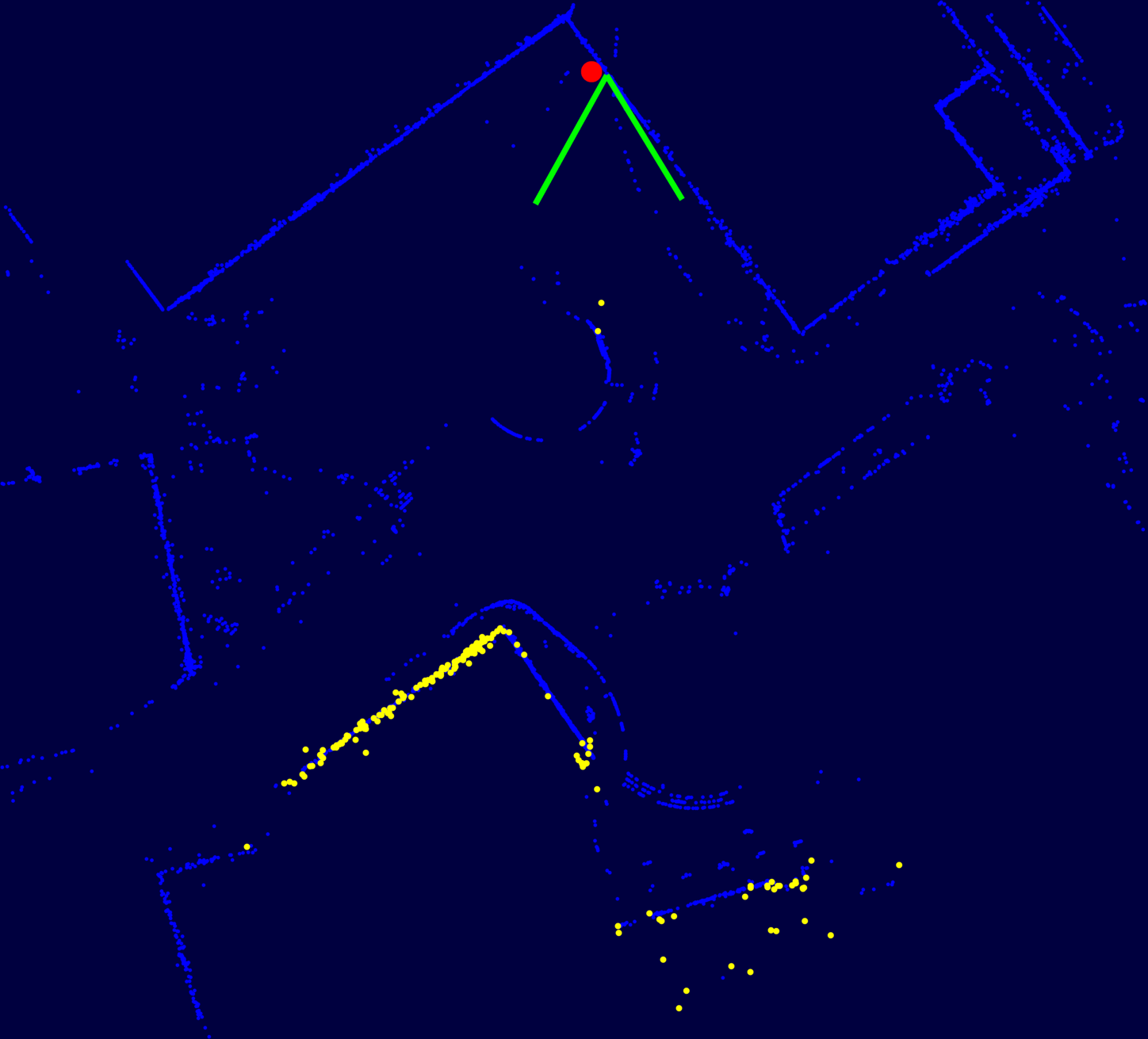} &

  \topinset{\includegraphics[width=0.07\paperwidth]{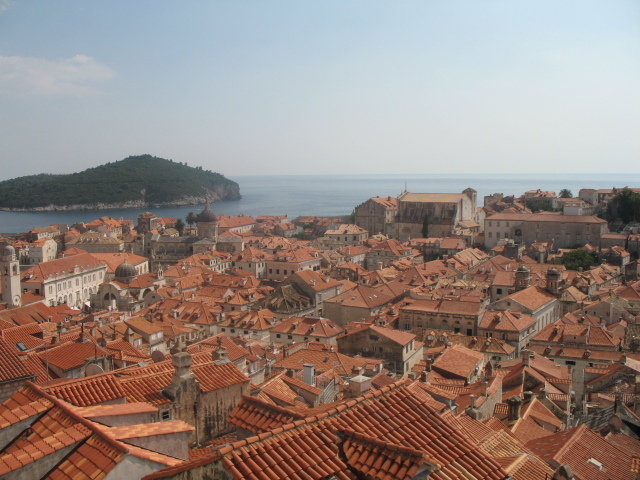}}
  { \includegraphics[width=0.16\paperwidth]{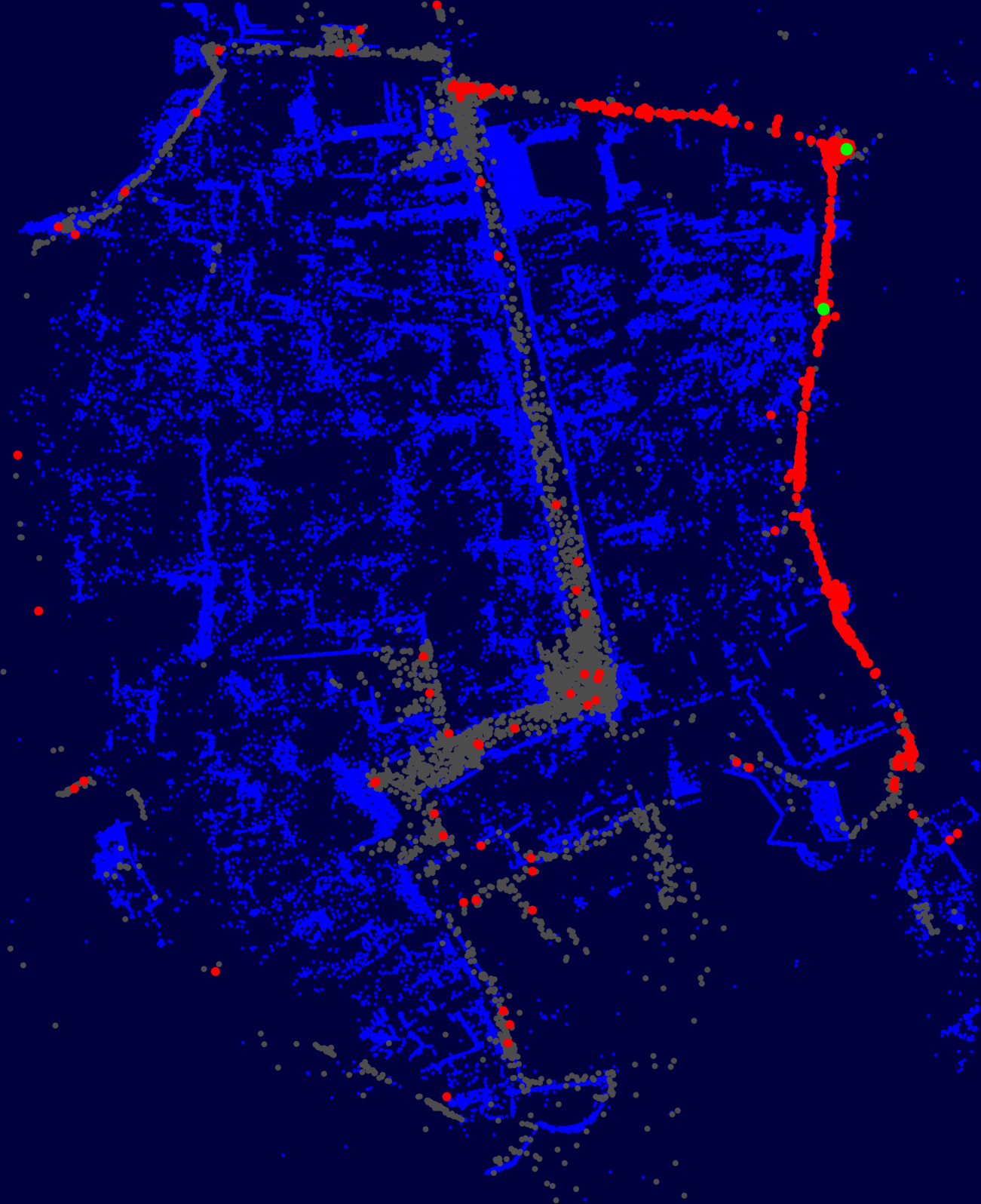} }{30pt}{-5pt} &
  \includegraphics[width=0.16\paperwidth]{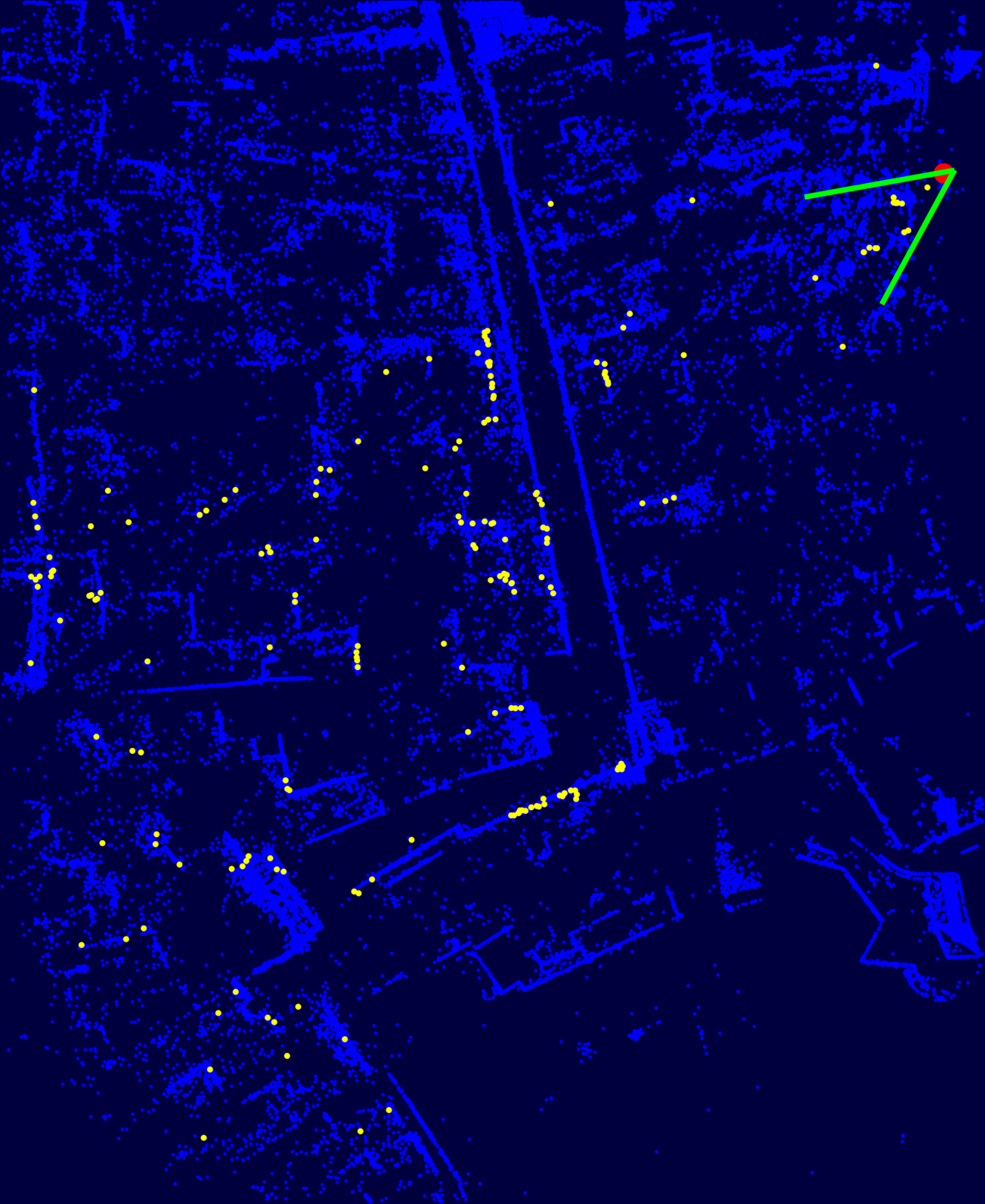} \\

  \topinset{\includegraphics[width=0.09\paperwidth]{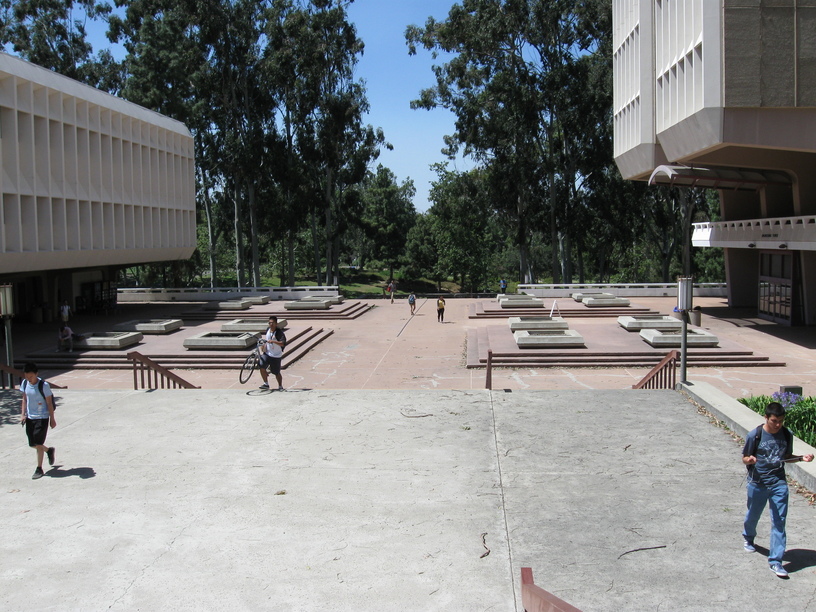}}
  {\includegraphics[width=0.2\paperwidth]{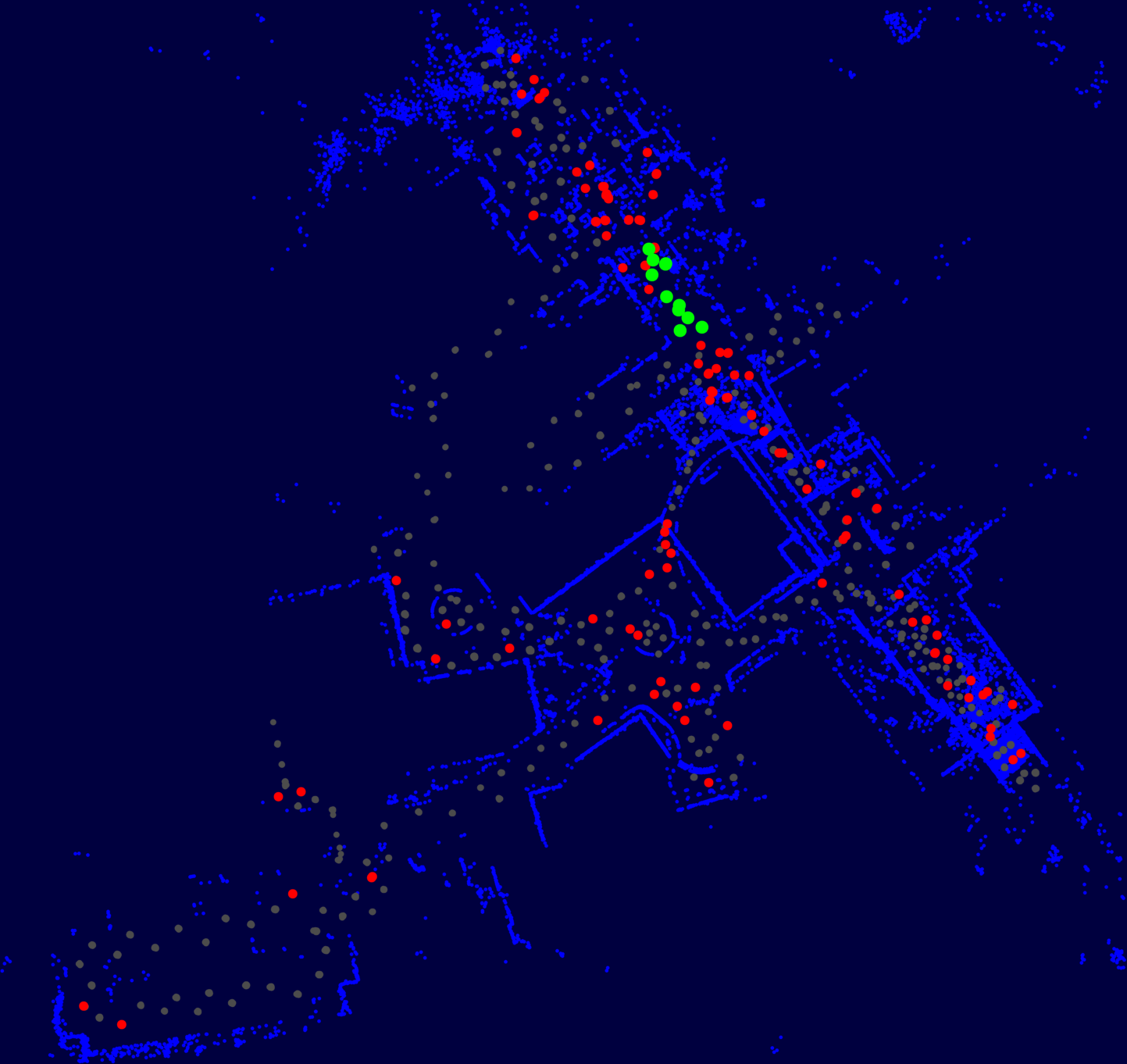}}{25pt}{-5pt} &
  \includegraphics[width=0.2\paperwidth]{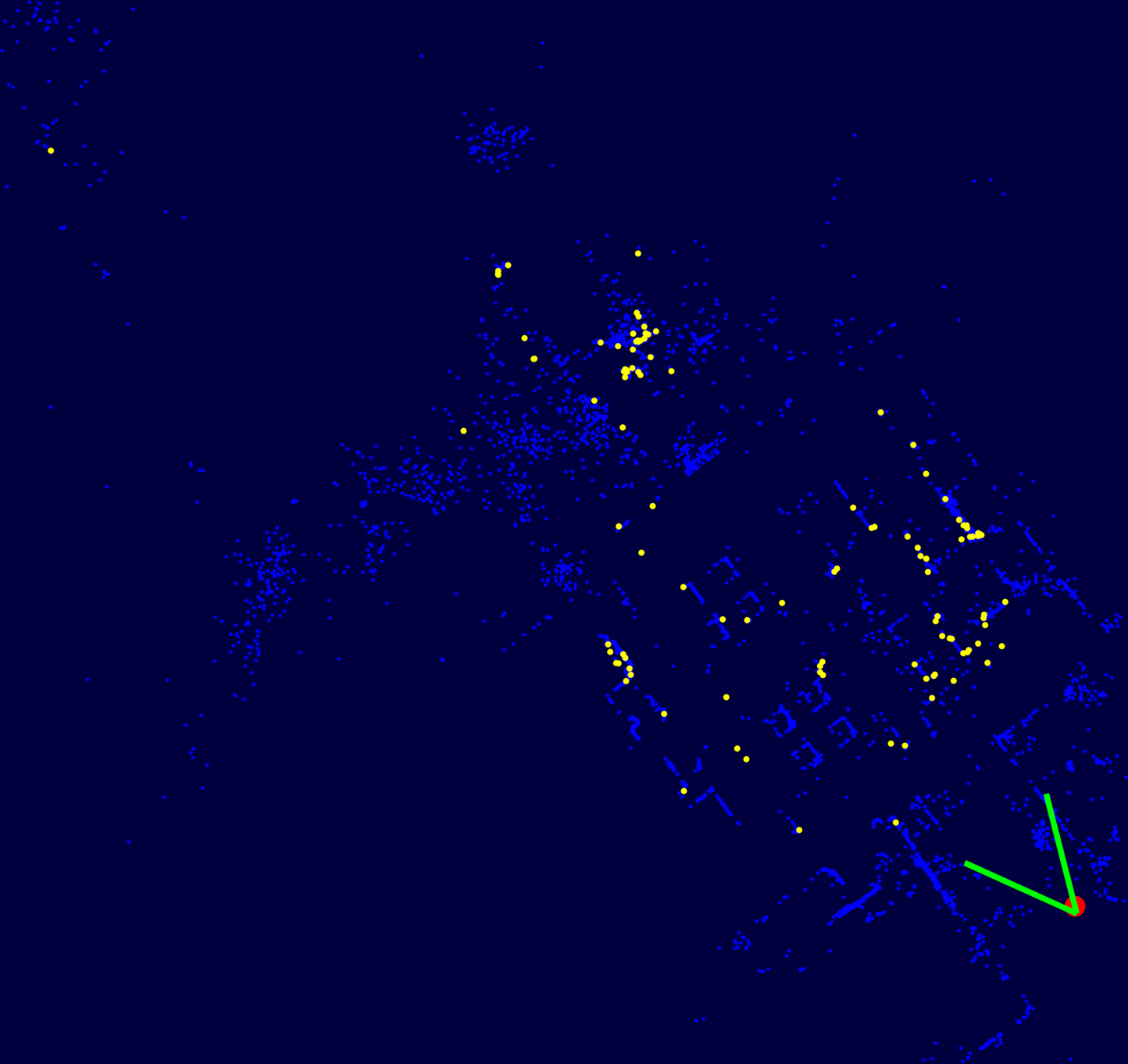} &

  \topinset{\includegraphics[width=0.06\paperwidth]{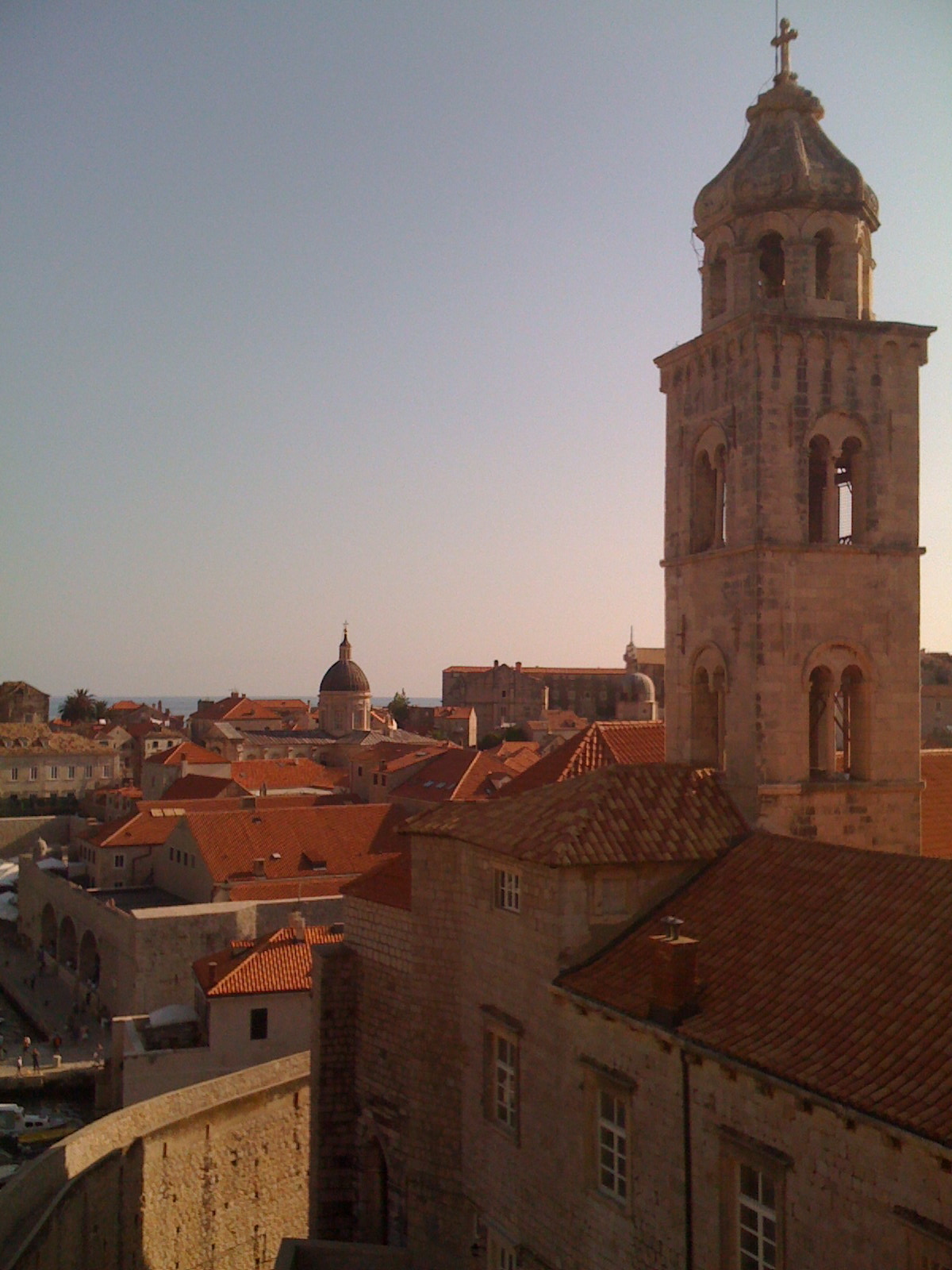}}
  { \includegraphics[width=0.16\paperwidth]{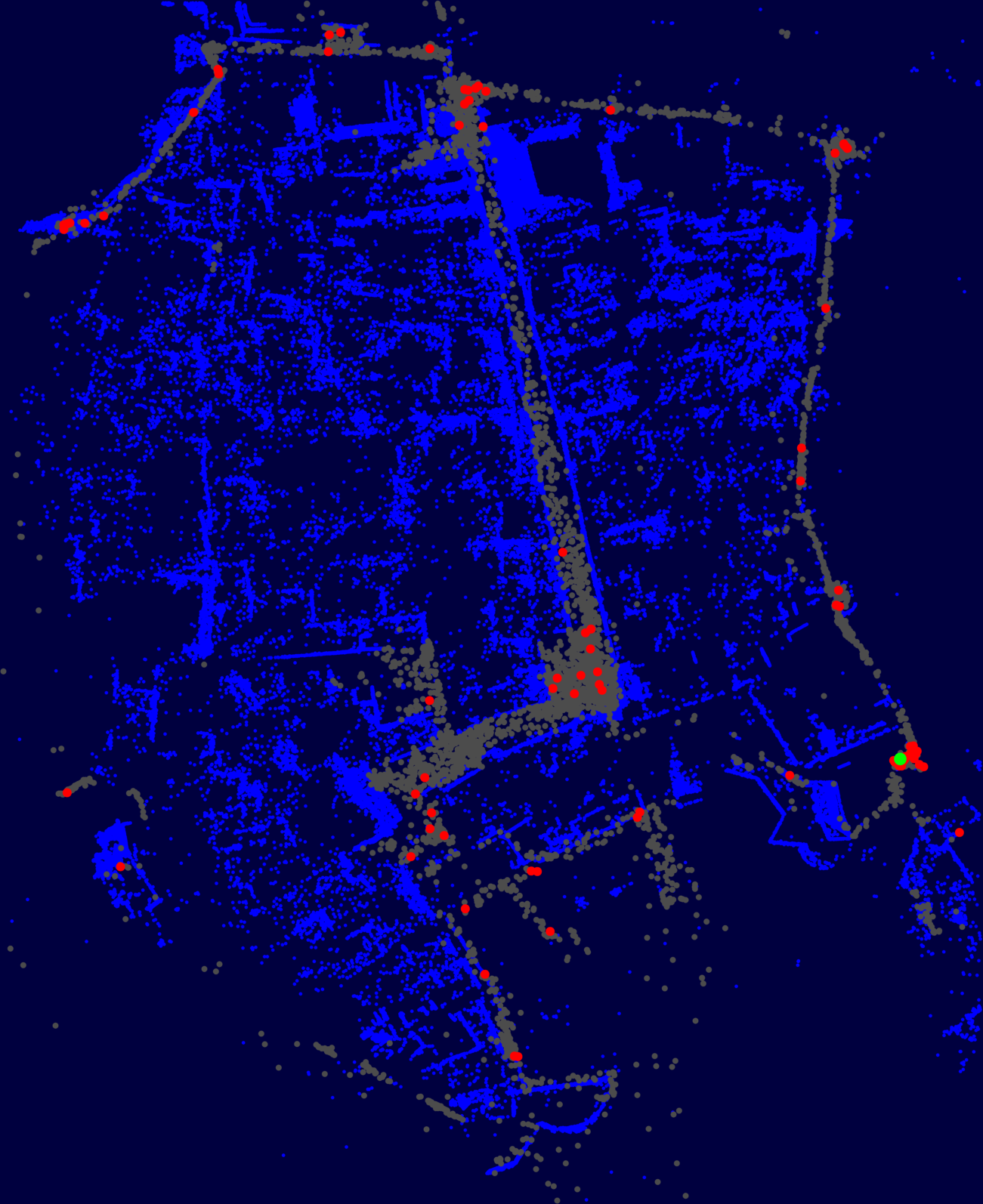}}{30pt}{-5pt} & 
  \includegraphics[width=0.162\paperwidth]{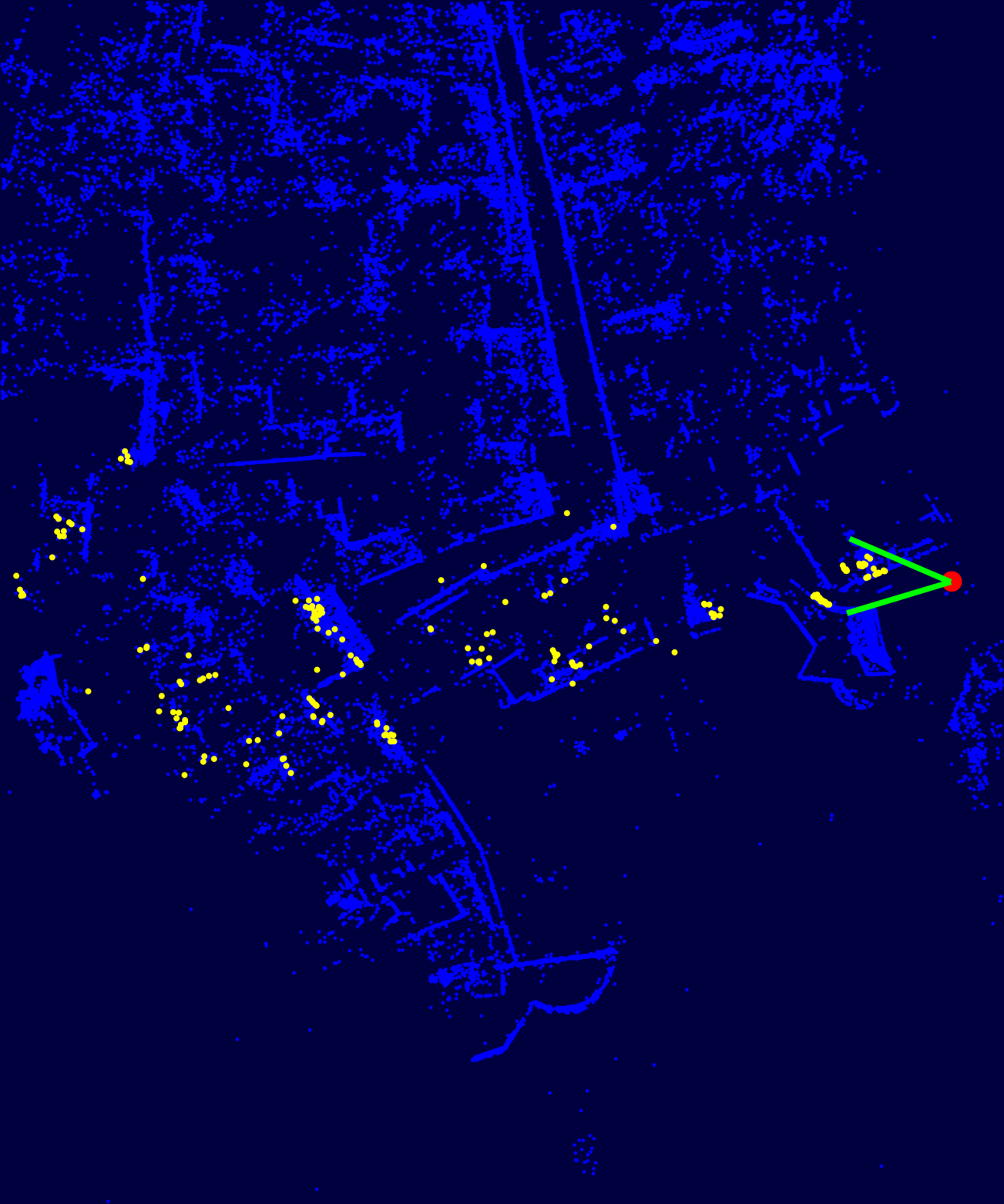} \\

  \multicolumn{2}{c}{\scriptsize Eng-Quad} &
  \multicolumn{2}{c}{\scriptsize Dubrovnik}

  \end{tabular}

  \vspace{.3cm}
  \caption[Qualitative localization results]{Qualitative localization results.
  Left column: model images (gray) are highlighted (red) if they received  
  one vote from Algorithm \ref{alg:fwdrt}. Algorithm \ref{alg:bckrt}
  quickly recognizes model images from the same general area of camera pose
  space (green). Right column: correspondences used by the PnP solver (yellow),
  along the localized camera (green). Ground truth camera position is indicated
  with a red circle. Best viewed zoomed in color.}
  \label{fig:reg}
  \vspace{-0.1in}
\end{figure*}

\paragraph{Anytime performance:} The runtime of our algorithm for camera
localization depends on two parameters: $N_{F}$ and $N_{B}$. Setting these
parameters trades off localization accuracy with faster execution times. Figure
\ref{fig:anytime} shows the influence of these variables using the Eng-Quad
dataset. We benchmarked forward matching times by randomly sampling query
features until a desired number $N_{F}$ pass the global ratio under fixed
values for $N_{B}$. Similarly, we fixed $N_{F}$ and evaluated different values
for $N_{B}$. In both cases, the range of values tested vary from 50 up to 500
matched features. Figure \ref{fig:anytime} shows the number of registered
images under these different configurations, and the time spent to achieve such
a level of performance.

\paragraph{Experimental details:} We tested our localization pipeline using the
following settings: for each dataset, we built a global kd-tree index using
all model view descriptors. We request $k=5$ nearest neighbors and check 128 leaves. We
set $\tau=0.7$ across all of our ratio tests. We set $N_{F}=200$ and
$N_{B}=200$ to provide a good balance between camera localization and 
execution time. Algorithm \ref{alg:bckrt} stops after 20 back-matched images,
which is a generous setting in these datasets (in most cases $N_{B}$ is achieved 
in less than 5 loops). Experiments were performed
using a single thread on an Intel i7-5930 CPU at 3.50GHz. We used the
implementation of \cite{Sattler2011,Sattler2012} in Eng-Quad and the re-bundled
Dubrovnik comparisons, running a single thread on an Intel i7-3770 CPU at 3.40GHz. 
We used a generic vocabulary tree and default parameters: $N_{t}=100$ for 
\cite{Sattler2011} and $N_{3D}=200$ for \cite{Sattler2012}. Unfortunately,
implementations of \cite{zeisl2015camera,svarm2016city} were not available.

\paragraph{Camera Localization:} We successfully localized all images in
Dubrovnik, except one image in the corrected version using P4Pf. We achieved
the smallest localization errors for all quartiles, and reported more images
within the $18.3m$ threshold and fewer beyond the $400m$ mark.  Despite finding
a substantial higher number of inliers, our method yielded larger average
errors with respect to the original Dubrovnik model due to its underlying
defects in the ground-truth.  \cite{svarm2016city} and \cite{zeisl2015camera}
(after RANSAC), who use a shape-voting approximation to the rough image
location rather than the traditional \textit{match-and-RANSAC} pipeline, report
smaller localization errors but at the cost of longer runtimes. Finally, we
successfully localized all query images in the Rome dataset using P4Pf. Rome
also suffers the similar inaccuracies as Dubrovnik, which resulted in the loss of
one test image using P3P.

The benefits of our approach are more pronounced
on Eng-Quad, due to its difficult characteristics. We localized more than 100 and 
50 additional cameras w.r.t. \cite{Sattler2011,Sattler2012} respectively, improving 
all localization errors except the first quartile using P4Pf. We obtain 
faster runtimes than \cite{li2010location,zeisl2015camera} while being competitive
with those of \cite{Sattler2011,Sattler2012}. Notably, our approach adapts
better to the more difficult Eng-Quad dataset, spending more time retrieving
images with sufficient correspondences. On the other hand, we quickly
recognize landmarks in Dubrovnik with the first or second top ranked images,
quickly retrieving sufficient putative correspondence and yielding faster
localization times.

\paragraph{Location Retrieval:} We obtained an asymptotic recall of 66.63\% on
the SF-0 dataset using the protocol of \cite{Li}.  At 95\% precision, the recall 
drops to 52.30\% using the \textit{effective inlier count} of \cite{Sattler2015}, 
falling below performance of other methods \cite{chen2011city, Sattler2015, arandjelovic2014,
zeisl2015camera} for location recognition.  For this test we used less stringent
parameters: $k=7$, $N_{F}=500$, $N_{B}=300$, and back matched up to 50 images. 
We expect tuning these parameters and utilizing re-ranking heuristics exploited by 
other methods to provide a better approach for such location retrieval
problems.  

\section{Conclusion}

Alternatives to large-scale image localization have focused on reducing the
density of the search space to quickly find discriminative correspondences.
Here we have shown that retrieving multiple global nearest neighbors and
filtering them using approximations to the ratio test can quickly identify
candidate regions of pose space. Such regions can be further refined by back
matching to yield state-of-the-art results in camera localization, even for
datasets with challenging global repeated structure. 

{\small
\bibliographystyle{ieee}
\bibliography{thesis}
}

\end{document}